%% file: paper.tex
\documentclass[twoside]{article}

\usepackage{mathpazo}
\usepackage{twemojis}

\usepackage[final]{manuscript}

\usepackage{eccvabbrv}
\usepackage[utf8]{inputenc}
\usepackage[T1]{fontenc}
\usepackage{hyperref}
\usepackage{url}
\usepackage{booktabs}
\usepackage{amsfonts}
\usepackage{nicefrac}
\usepackage{microtype}
\usepackage{xcolor}
\usepackage{amsmath,amsfonts,amsthm,bm}
\usepackage{cleveref}
\usepackage{thm-restate}
\usepackage{graphicx}
\usepackage{algorithm}
\usepackage[noend]{algpseudocode}
\usepackage{multicol}
\usepackage{multirow}
\usepackage[percent]{overpic}
\usepackage{caption}
\usepackage{subcaption}
\usepackage{listings}
\usepackage{float}
\usepackage[round]{natbib}
\usetikzlibrary{calc}

\newcommand{\ours}{S$^{3}$}

\newtheorem{theorem}{Theorem}

\newtheorem{definition}{Definition}
\newtheorem{example}[definition]{Example}
\theoremstyle{remark}
\newtheorem*{remark}{Remark}

\input{./commands.tex}

\hypersetup{
    colorlinks=true,
    linkcolor=blue,
    filecolor=magenta,      
    urlcolor=cyan,
    citecolor=blue
}

\title{S\texorpdfstring{$^3$}{3}: Structured Sparsity Specification}
\author{Ayoub Ghriss}

\addauthor{Ayoub Ghriss}{Department of Computer Science\\Colorado University, Boulder}{ayoub.ghriss@polytechnique.org}
\begin{document}

\maketitle

\begin{abstract}
  \input{./0_abs.tex}
\end{abstract}

\section{Introduction} \label{sec:intro}
\input{./1_intro.tex}

\section{Related Work} \label{sec:related}
\input{./2_related.tex}

\section{Preliminaries} \label{sec:prelims}
\input{./3_prelims.tex}

\section{The S\texorpdfstring{$^3$}{3} Framework} \label{sec:framework}
\input{./4_framework.tex}

\section{Integration with Pruning Algorithms} \label{sec:integrations}
\input{./5_integrations.tex}

\
\section{Experiments} \label{sec:experiments}
\input{./6_experiments.tex}

\section{Conclusion} \label{sec:conclusion}
\input{./7_conc.tex}

\bibliography{references}

\newpage
\onecolumn
\appendix
\input{./appendix.tex}

\end{document}

%% file: commands.tex
\newcommand{\defeq}{\stackrel{\text{def}}{=}}

\newcommand{\normf}[1]{\left\| #1 \right\|_{\mathrm{F}}}

\def\eqref#1{equation~\ref{#1}}

\def\1#1{\bm{1}_{#1}}

\def\vb{{\bm{b}}}

\def\vd{{\bm{d}}}

\def\vi{{\bm{i}}}
\def\vj{{\bm{j}}}

\def\vo{{\bm{o}}}

\def\vs{{\bm{s}}}

\def\vw{{\bm{w}}}

\DeclareMathAlphabet{\mathsfit}{\encodingdefault}{\sfdefault}{m}{sl}
\SetMathAlphabet{\mathsfit}{bold}{\encodingdefault}{\sfdefault}{bx}{n}

\def\sZ{{\mathbb{Z}}}

\newcommand{\R}{\mathbb{R}}

\newcommand{\T}{\mathcal{T}}
\newcommand{\B}{\mathcal{B}}
\newcommand{\Scope}{\mathcal{S}}
\newcommand{\Layout}{\mathcal{L}}
\newcommand{\View}{\mathcal{V}}
\newcommand{\Z}{\mathbb{Z}}
\newcommand{\N}{\mathbb{N}}
\newcommand{\vW}{\mathbf{W}}
\newcommand{\vH}{\mathbf{H}}
\newcommand{\dom}{\text{dom}}
\newcommand{\cosize}{\text{cosize}}
\newcommand{\idx}{\text{idx}}
\newcommand{\Elements}{\text{Elements}}
\newcommand{\Blocks}{\text{Blocks}}

\newcommand{\CoupledBlock}{\text{CoupledBlock}}
\newcommand{\CoupledScope}{\text{CoupledScope}}

\newcommand{\sthree}{{S$^3$}}

%% file: 0_abs.tex
We introduce the \emph{Structured Sparsity Specification} (\sthree), an algebraic framework for
defining, composing, and implementing structured sparse patterns. \sthree{} specifies sparsity
through three components: a \emph{View} that reshapes the tensor via layout composition, a
\emph{Block} specification that defines the atomic pruning unit, and the sparsity decision
\emph{Scope}. Both Block and Scope support \emph{Coupling} across tensors for coordinated
sparsification. \sthree{} enables precise specification of diverse sparsity structures, from
fine-grained N:M patterns to coarse channel pruning, and integrates seamlessly with Optimal Brain
Damage (OBD) and Surgeon (OBS). We formalize the framework mathematically, demonstrate its
expressiveness on canonical patterns, and validate it experimentally via structured OBS and OBD
implementations\footnote{\href{https://sparsekit.readthedocs.io}{\texttt{https://sparsekit.readthedocs.io}}}
built entirely on \sthree{}, which surpasses well-established second order heuristics on output
reconstruction across common configurations.

%% file: 1_intro.tex
\begin{figure*}[t]
    \centering
  \begin{subfigure}[b]{0.3\linewidth}
    \centering
    \includegraphics[width=\linewidth]{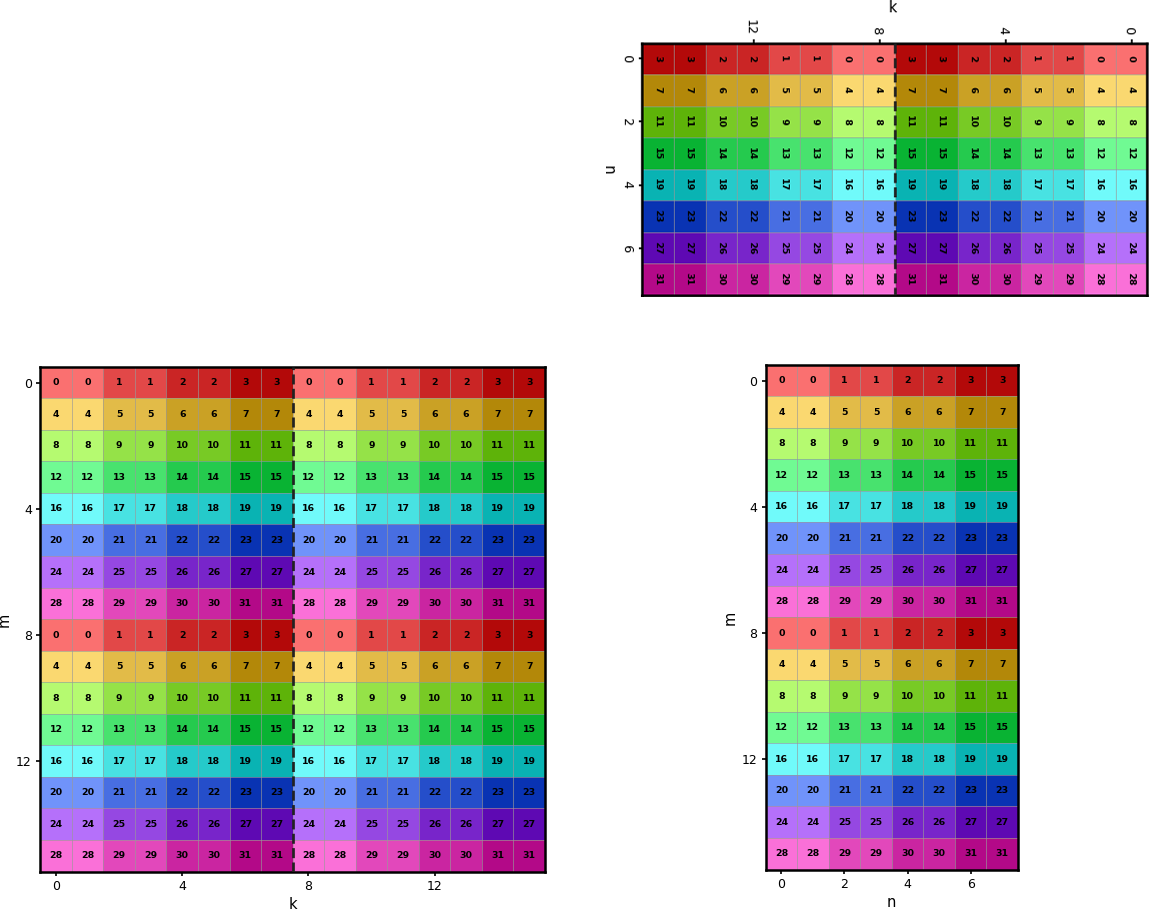}
    \caption{NVIDIA Ampere tensor cores partition the left multiplicand ($16\times 16$) and right multiplicand ($16\times 8$) so that each thread-group of 4 threads provides one row/column of the left/right operand.}
    \vspace{1em}
    \label{fig:wmma_layout}
  \end{subfigure}\hspace{1em}
  \begin{subfigure}[b]{0.5\linewidth}
    \centering
    \includegraphics[width=\linewidth]{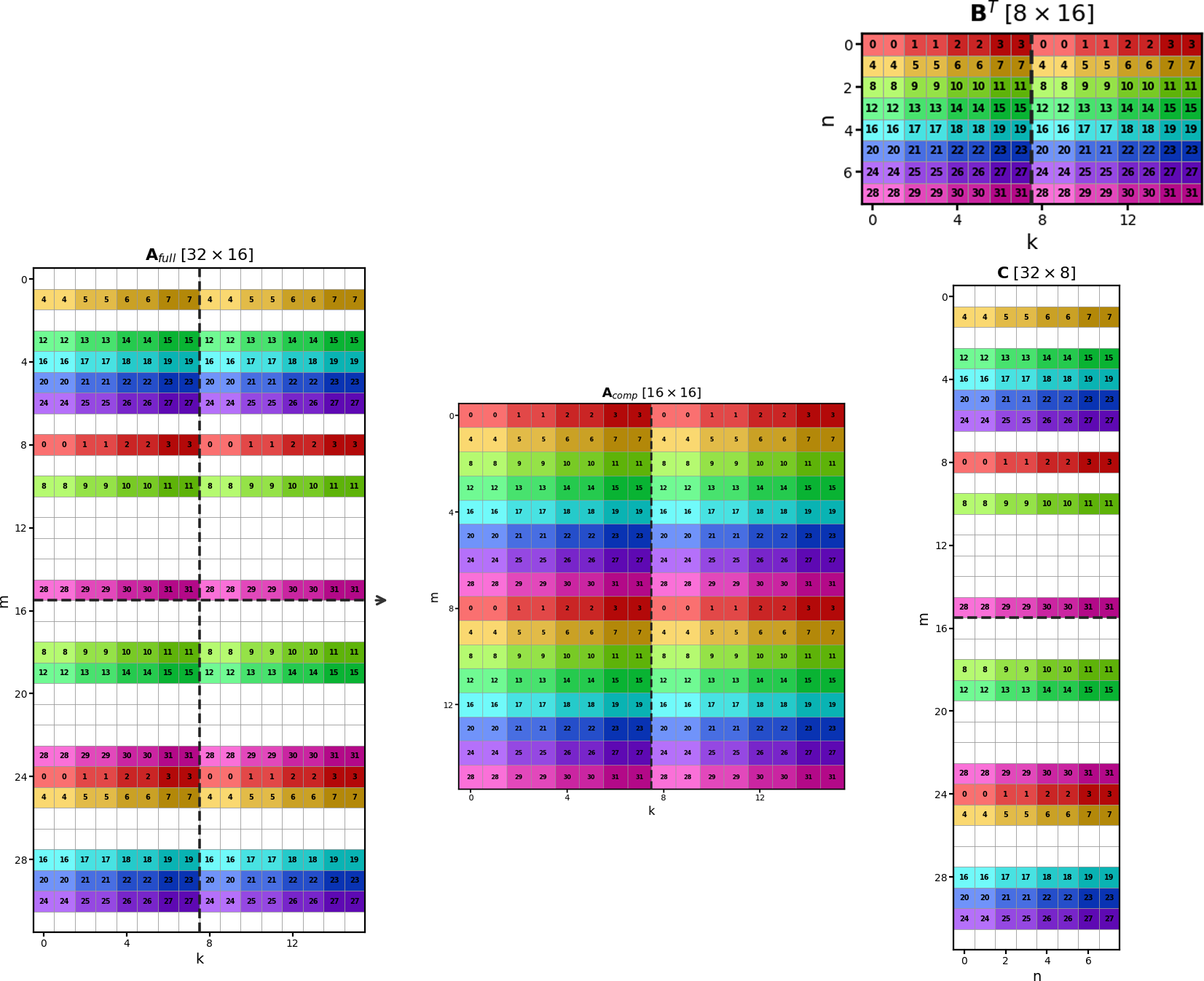}
    \caption{By structuring the sparsity so that each thread-group handles the same exact number of rows per $32\times 16$ left multiplicand, we design a new sparsity pattern that is hardware-native without the need for special hardware support.}
    \label{fig:sparse_wmma}
   \end{subfigure}
   \caption{\sthree{} can target low-level hardware patterns to avoid any slowdowns due to memory access conflicts. For \texttt{m16n8k16} matrix multiplication patterns, we can target the thread-group as one block.}
\end{figure*}
Quantization has become the dominant approach to compressing large neural networks, delivering
substantial inference savings with minimal accuracy loss. Sparsity offers a complementary axis of
compression that can compose with quantization for compounded gains, yet its practical impact has
been limited. Unstructured pruning achieves high theoretical compression ratios, but its irregular
memory access patterns and workload imbalance often negate these benefits on modern
hardware~\citep{structuredsurveyHe24}. Structured sparsity, which removes coherent blocks of
parameters such as channels, filters, or attention heads, is more hardware-friendly but
traditionally lacks the flexibility to express complex patterns, limiting its applicability.

Moreover, the current landscape of structured sparsity is fragmented with the "structure"
qualification used to describe specific patterns: N:M methods target specific hardware (NVIDIA
sparse tensor cores), block sparsity aligns with matrix multiplication tiles, and channel pruning
operates at the semantic level of neural network layers. Each paradigm requires its own pruning
implementation, and no common formalism exists to relate them or to develop algorithms that
generalize across patterns.

We address this gap by introducing Structured Sparsity Specification (\sthree{}), a framework that
unifies diverse sparsity patterns under a common algebraic formalism, composes simple
specifications into complex hierarchical structures, and integrates directly with second-order
pruning criteria through well-defined index mappings. Our contributions:
\begin{itemize}
  \setlength\itemsep{0em}
  \item A formal specification language based on layout algebra that can express any hyperrectangular
        sparsity pattern through View, Block, and Scope primitives.
  \item Structured extensions of OBD and OBS (S-OBD, S-OBS) with derived blockwise saliency scores, optimal
        weight updates, and efficient Schur complement maintenance.
  \item Experimental validation showing that S-OBS implementation, with specification-agnostic code,
        outperforms SparseGPT by 16--20\% across four distinct sparsity configurations.
\end{itemize}

%% file: 2_related.tex
Neural network pruning has evolved from simple magnitude-based methods~\citep{han2015learning} to
principled second-order approaches. Optimal Brain Damage (OBD)~\citep{lecun1989optimal} uses
diagonal Hessian approximations to score parameter importance, while Optimal Brain Surgeon
(OBS)~\citep{hassibi1993optimal} incorporates the full inverse Hessian and derives optimal weight
updates that compensate for removed parameters. The Lottery Ticket
Hypothesis~\citep{frankle2019lottery} provided further motivation by showing that dense networks
contain sparse subnetworks matching full accuracy, suggesting that most parameters are redundant.
For large language models, where retraining is prohibitively expensive, Wanda~\citep{sun2023wanda}
and SparseGPT~\citep{frantar2023sparsegpt} adapt OBD and OBS principles to the zero-shot setting
with efficient Hessian approximations.

A parallel line of work has explored structured sparsity constraints that can be exploited by
hardware. N:M sparsity enforces exactly $N$ non-zero values per $M$ consecutive elements, directly
mapping to NVIDIA's sparse tensor cores~\citep{nvidia2020ampere}. Block sparsity zeros entire
contiguous blocks, enabling efficient dense operations on the surviving blocks~\citep{gray2017gpu}.
Channel and filter pruning remove entire convolutional channels or attention
heads~\citep{li2016pruning,michel2019sixteen}. Early work on structured sparse
coding~\citep{structsparsemairal11a} used sums of $\ell_\infty$
norms~\citep{tibshirani1996regression, chen1998atomic, mallat1999wavelet} to enforce scope-level
consistency, establishing a principled framework to design arbitrary patterns that can be adapted
to today's hardware-aligned structures. Despite this variety, each pattern typically requires its
own pruning implementation; \sthree{} provides a unified specification that covers all of these as
special cases.

Dynamic sparse training~\citep{mocanu2018scalable, evci2020rigging} takes a different approach:
instead of pruning a trained model, it maintains a fixed sparsity budget throughout training while
periodically updating the mask. These methods demonstrate that static masks are suboptimal, but the
resulting patterns are unstructured and cannot leverage hardware-accelerated sparsity formats. We
discuss how \sthree{} can serve as a structural backbone for dynamic sparse training.

Our sparsity formalism was inspired by the CuTe library~\citep{nvidia23cute}, which implements a
layout algebra \citep{colfax25layout} of shape-stride pairs with composition operations for GPU
kernel design. We extend this representation from the domain of memory addressing to sparsity
specification, using layout composition to define views over dense tensors that expose block and
scope structure.

%% file: 3_prelims.tex
Layout Algebra~\citep{nvidia23cute}; see~\citet[Chapter~2]{colfax25layout} for a comprehensive
treatment provides a rigorous tensor layout formalism that we leverage in our framework.

A \textbf{layout} $\Layout = \vs:\vd$ with shape $\vs\in \N^n$ and stride $\vd\in\Z^n$ is a mapping
from a logical coordinate space to linear memory indices.

\begin{definition}[Layout Function]
  \label{def:layout_function}
  For a layout $\Layout = \vs:\vd$ with $n$ dimensions, the index function $\phi_\Layout: \Z^n \to \Z$ is:
  \begin{equation}
    \phi_\Layout(i_0, \ldots, i_{n-1}) = \sum_{k=0}^{n-1} i_k \cdot d_k
  \end{equation}
  where $0 \leq i_k < s_k$ for all $k$.
\end{definition}

\begin{definition}[Layout Size]
  \label{def:layout_size}
  The size of layout $\Layout$ is defined as:
  \begin{equation}
    |\Layout| = \prod_{k=0}^{n-1} s_k
  \end{equation}
\end{definition}

\begin{definition}[Layout Cosize]
  \label{def:layout_cosize}
  The cosize represents the memory footprint:
  \begin{equation}
    \cosize(\Layout) = \max_{\vi \in \dom(\Layout)} \phi_\Layout(\vi) + 1 = \sum_{k=0}^{n-1} (s_k - 1) \cdot d_k + 1
  \end{equation}
\end{definition}

\begin{definition}[Layout Composition]
  \label{def:layout_composition}
  Given layouts $\mathcal{A} = \vs_A:\vd_A$ and $\B = \vs_B:\vd_B$ where $|\mathcal{A}| = |\B|$, the composition $\mathcal{A} \circ \B$ produces a new layout where $\B$ indexes into the logical space of $\mathcal{A}$:
  \begin{equation}
    \phi_{\mathcal{A} \circ \B}(\vj) = \phi_\mathcal{A}(\phi_\B^{-1}(\vj))
  \end{equation}
\end{definition}


The key takeaway is that any tensor can be described by a layout $\Layout = \vs:\vd$ that may
itself be the composition or division of simpler layouts. This compositionality is central to
\sthree{}: it allows us to factor the View out of the Block and Scope specifications.

%% file: 4_framework.tex
\begin{figure*}[!t]
  \centering
  \begin{subfigure}[b]{0.22\linewidth}
    \centering
    \includegraphics[width=\linewidth]{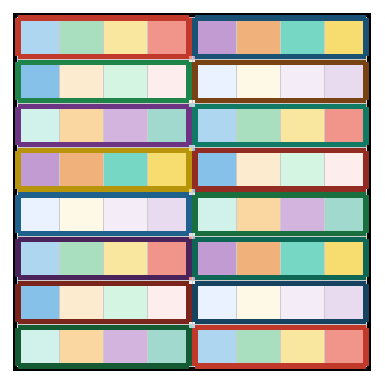}\\[2pt]
    \includegraphics[width=\linewidth]{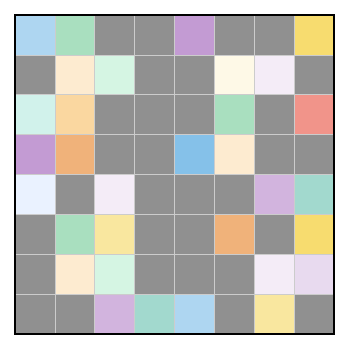}
    \caption{2:4 Sparsity. $\View\!=\!\Layout_{\text{phys}}$,
      $\vb\!=\!(1,1)$, $\vs\!=\!(1,4)$, $k\!=\!2$.
      Scalar blocks; scopes of 4 consecutive columns.}
    \label{fig:app_pat_24}
  \end{subfigure}\hfill
  \begin{subfigure}[b]{0.22\linewidth}
    \centering
    \includegraphics[width=\linewidth]{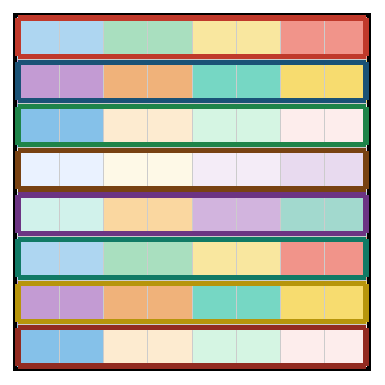}\\[2pt]
    \includegraphics[width=\linewidth]{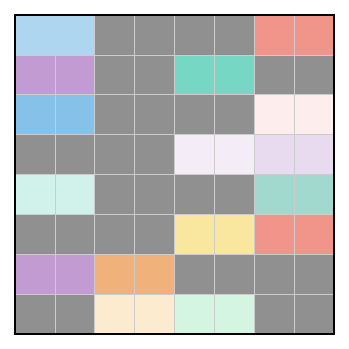}
    \caption{4:8 Sparsity. $\View\!=\!\Layout_{\text{phys}}$,
      $\vb\!=\!(1,2)$, $\vs\!=\!(1,4)$, $k\!=\!2$.
      Column-pair blocks; scopes of 4 blocks.}
    \label{fig:app_pat_48}
  \end{subfigure}\hfill
  \begin{subfigure}[b]{0.2\linewidth}
    \centering
    \includegraphics[width=\linewidth]{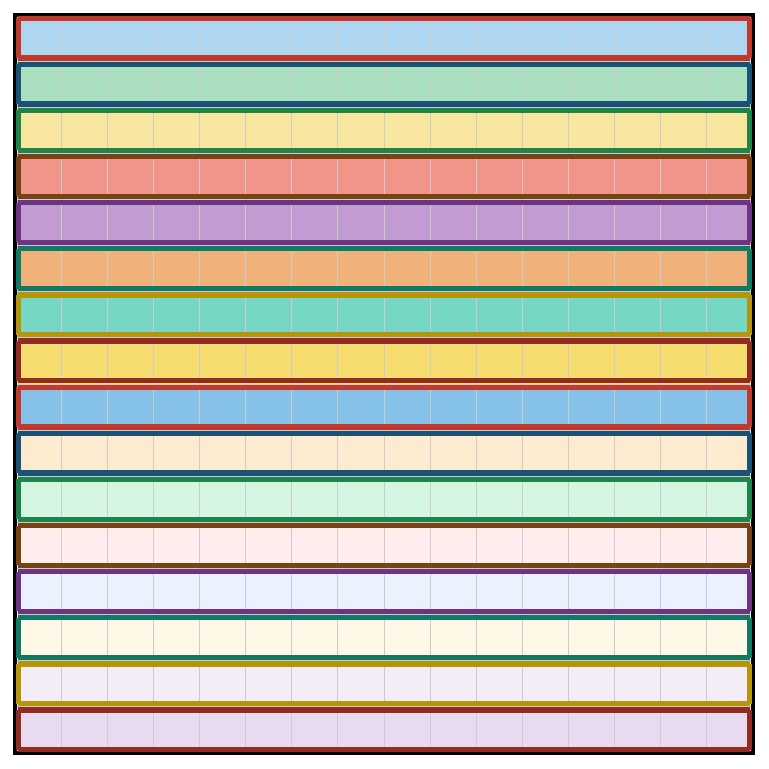}\\[2pt]
    \includegraphics[width=\linewidth]{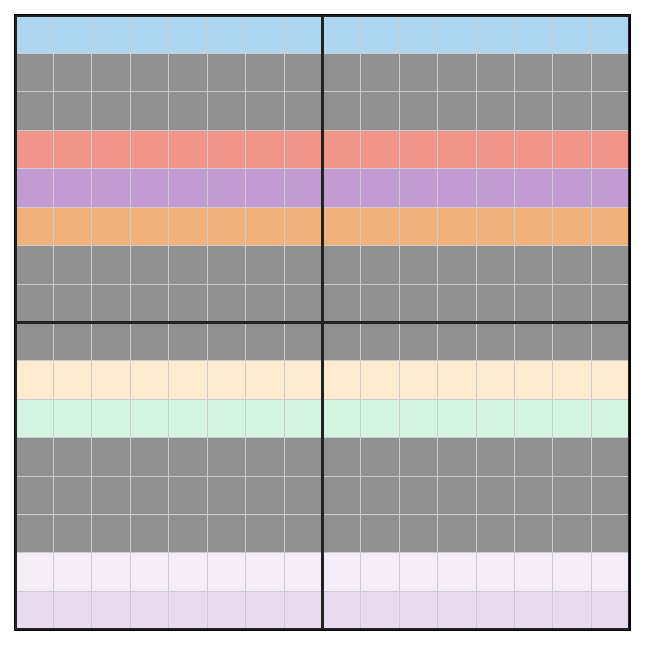}
    \caption{1:2 thread-group coupled sparsity as seen in~\Cref{fig:sparse_wmma},
      $\View\!=\!(2,8,16):(128,\!16,\!1)$,
      $\vb\!=\!(1,1,16)$, $\vs\!=\!(2,1,1)$, $k\!=\!1$.
      16-col blocks; rows 8 apart compete.}
    \label{fig:app_pat_cr}
  \end{subfigure}\hfill
  \begin{subfigure}[b]{0.27\linewidth}
    \centering
    \includegraphics[width=\linewidth]{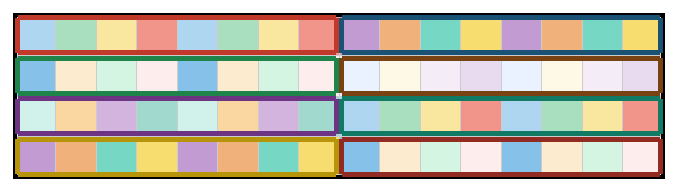}\\[2pt]
    \includegraphics[width=\linewidth]{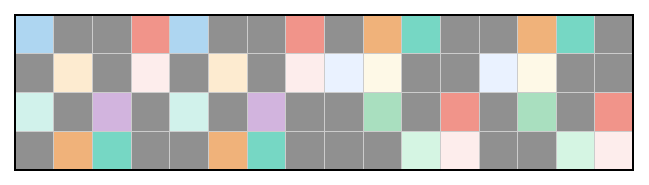}
    \caption{Coupled 2:4.
      $\View\!=\!(M,\frac{K}{8},4,2):(K,8,1,4)$,
      $\vb\!=\!(1,1,1,2)$, $\vs\!=\!(1,1,4,1)$, $k\!=\!2$.
      Column-pair blocks couple cols 4 apart to reduce metadata.}
    \label{fig:app_pat_c24}
  \end{subfigure}
  \caption{Experimental sparsity patterns: structure (top) and pruned result (bottom).
    Cell fill color encodes block identity; colored outlines delimit scopes;
    dark gray cells are pruned to zero.
    All four patterns are expressed by the same three primitives
    (View, Block, Scope) with different shapes, and pruned by the same
    S-OBS algorithm.}
  \label{fig:patterns_app}
\end{figure*}

\sthree{} specifies structured sparsity through three components:
\begin{enumerate}
  \item \textbf{View}: A layout that reshapes the tensor domain for sparsity operations.
  \item \textbf{Block}: The atomic sparsity unit, i.e., the smallest
        set of parameters that are pruned together.
  \item \textbf{Scope}: Defines the sparsity decision scope, i.e., the set of
        blocks over which a sparsity level is enforced.
\end{enumerate}

Both Blocks and Scopes support \textbf{Coupling} to coordinate pruning across multiple tensors.
\sthree{} can operate on any subset of the tensor's domain; elements outside the domain remain
dense. For ease of exposition, we omit the domain from the specification and treat the target
domain as a virtual tensor.

\begin{remark}[View Factorization]
  \label{rem:view_factorization}
  In principle, both Block and Scope admit independent view layouts $\View_B$ and $\View_G$.
  However, by layout composition (Definition~\ref{def:layout_composition}), $\View_G \circ
    \View_B$ yields a single view $\View$. We therefore factor the view out of both
  specifications without loss of generality, resulting in a cleaner three-component
  decomposition: View, Block, Scope.
\end{remark}

\subsection{View}
\label{sec:view}

The View defines how the underlying tensor is logically reorganized before block and scope
boundaries are imposed.

\begin{definition}[View]
  \label{def:view}
  Given a tensor $\T$ with physical layout $\Layout_{\text{phys}} = \vs_{\text{phys}}:\vd_{\text{phys}}$, the View $\View$ is a layout satisfying:
  \begin{equation}
    |\View| = |\Layout_{\text{phys}}|
  \end{equation}
  The view defines a coordinate transformation $\psi: \dom(\View) \to \dom(\Layout_{\text{phys}})$ mapping logical coordinates to physical tensor coordinates.
\end{definition}

\begin{example}
  For a matrix $\vW \in \R^{64 \times 64}$ (see~\Cref{fig:block_sparsity}) with row-major storage $\Layout_{\text{phys}}
    = (64, 64):(64, 1)$, a view for $16\times 16$ block-wise sparsity:
  \begin{equation}
    \View = (4,4,16,16):(512,16,64,1)
  \end{equation}
  This reshapes the matrix into a $4\times 4$ grid of $16\times 16$ blocks, so that the block specification can be expressed easily.
\end{example}

\subsection{Block Specification}
\label{sec:block}

The Block Specification (\texttt{BlockSpec} in the implementation, ~\Cref{appendix:implementation})
represents the extent of each block within the viewed space.

\begin{definition}[Block Shape]
  \label{def:block_shape}
  Given View $\View$ with $n$ dimensions, the Block Shape $\vb = (b_0, b_1, \ldots, b_{n-1})$ satisfies:
  \begin{equation}
    s_k^{(\View)}~\equiv 0 \mod{b_k}\quad, \forall k \in \{0, \ldots, n-1\}
  \end{equation}
  where $s_k^{(\View)}$ is the $k$-th shape component of $\View$.
\end{definition}

The block layout $\Layout_B$ inherits strides from the view: $\Layout_B \defeq \vb:\vd^{(\View)}$,
as such, the Block is simply characterized by its shape $\vb\in\sZ^{n}$.

\begin{definition}[Block Grid]
  \label{def:block_grid}
  The block grid layout is:
  \begin{equation}
    \Layout_{\text{grid}} = \vs^{(B)}:\vd_{\text{grid}},
  \end{equation}
  where the number of blocks along each dimension is $g_k^{(B)} \defeq s_k^{(\View)} / b_k$, and the stride
  to the next dimension $d_{\text{grid},k} \defeq b_k \cdot d_k^{(\View)}$.
\end{definition}

All blocks are assumed to have the same size $|\B| = \prod_{k=0}^{n-1} b_k$ . When the tensor
dimensions are not evenly divisible, the view can pad the domain so that all blocks are uniform.
Since the block is the atomic sparsity unit, \ours{} is agnostic to the choice of view.

\begin{example}
  For the $16\times 16$ block view above, block shape $(1,1,16,16)$ yields a squeezed
  block grid of shape $(8,8)$. The grid strides are $(512,16)$, giving the starting
  address of each block whose elements are strided via $(64,1)$.
\end{example}

\subsection{Scope Specification}
\label{sec:scope}

Scopes organize blocks for sparsification decisions. The Scope operates on the block grid: blocks
are to the Scope what elements are to the Block.

\begin{definition}[Scope Shape]
  \label{def:scope_shape}
  The Scope Shape $\vs = (g_0, g_1, \ldots, g_{m-1})$ defines the extent of each scope over the block
  grid, satisfying divisibility constraints analogous to Block Shape. The number of blocks per scope:
  \begin{equation}\label{eq:block_size_blocks}
    |\Scope|_B = \prod_{k=0}^{m-1} g_k
  \end{equation}
  %
\end{definition}

\begin{definition}[Block-to-Scope Mapping]
  \label{def:block_to_scope}
  The function $\gamma: [0, |\Layout_{\text{grid}}|) \to [0, |\Layout_{\text{grid}}|/|\Scope|_B)$ maps block index to scope index.
\end{definition}

\begin{definition}[Scope Block Enumeration]
  \label{def:scope_block_enumeration}
  For scope $\ell$, the set of block indices is:
  \begin{equation}
    \Blocks(\ell) = \{ j : \gamma(j) = \ell \}
  \end{equation}
\end{definition}

\begin{definition}[Scope Element Enumeration]
  \label{def:scope_element_enumeration}
  For scope $\ell$, the set of element indices is:
  \begin{equation}
    \Elements_G(\ell) = \bigcup_{j \in \Blocks(\ell)} \Elements(j)
  \end{equation}
\end{definition}

\subsection{Coupling}
\label{sec:coupling}



Coupling coordinates sparsity decisions across multiple tensors so that they are pruned jointly.
This is essential for maintaining structural consistency; for example, matching input and output
channels, or pruning attention heads across Q, K, V, and O projections.

Each coupling is defined by a permutation $\pi^{(i)}$ per tensor that reorders the block grid
dimensions. The permutation brings the dimension to be coupled into the last position, so that
blocks at the same grid coordinate (after permutation) are concatenated across tensors.

\begin{definition}[Coupled Block Specification]
  \label{def:coupled_block}
  Given $K$ tensors with block specifications $\B^{(1)}, \ldots, \B^{(K)}$
  and permutations $\pi^{(1)}, \ldots, \pi^{(K)}$, the coupled specification
  $\B_\otimes = \B^{(1)} \otimes \cdots \otimes \B^{(K)}$ requires that the
  permuted grid shapes agree:
  \begin{equation}
    \vs^{(B,1)}_{\pi^{(1)}} = \vs^{(B,2)}_{\pi^{(2)}} = \cdots = \vs^{(B,K)}_{\pi^{(K)}}
  \end{equation}
  where $\vs^{(B,i)}_{\pi^{(i)}}$ denotes the grid shape of $\B^{(i)}$ with
  dimensions reordered by $\pi^{(i)}$.
\end{definition}

For a grid coordinate $\boldsymbol{\ell}$ in the common permuted grid, the coupled block
concatenates elements across tensors:
\begin{equation}
  \CoupledBlock(\boldsymbol{\ell}) = \bigcup_{i=1}^{K}
  \Elements^{(i)}\!\left(\pi^{(i)^{-1}}(\boldsymbol{\ell})\right)
\end{equation}
with coupled block size
$|\B_\otimes| = \sum_{i=1}^{K} |\B^{(i)}|$. Blocks at the same position
compete jointly during pruning.

\begin{definition}[Coupled Scope Specification]
  \label{def:coupled_scope}
  Given $K$ scope specifications $\Scope^{(1)}, \ldots, \Scope^{(K)}$ with
  permutations $\pi^{(1)}, \ldots, \pi^{(K)}$ aligning the scope grids, the
  coupled block at position $\boldsymbol{\ell}$ contains all blocks from
  every tensor at that coordinate:
  \begin{equation}
    \CoupledScope(\boldsymbol{\ell}) = \bigcup_{i=1}^{K} \Blocks^{(i)}\!\left(\pi^{(i)^{-1}}(\boldsymbol{\ell})\right)
  \end{equation}
  Within each coupled block, $k$-sparsity is applied over the concatenated
  block saliencies from all tensors.
\end{definition}

\subsection{Experimental Sparsity Patterns}
\label{sec:exp_patterns}

We illustrate how \sthree{} encodes the four sparsity patterns used in our experiments
(\Cref{sec:experiments}). Additional canonical patterns (unstructured, channel, head pruning,
partial tensor) appear in \Cref{app:patterns}. For a weight matrix $\vW \in \R^{M \times K}$:

\textbf{2:4 sparsity.}
The view is the physical layout ($\View = \Layout_{\text{phys}}$).
Scalar blocks ($\vb = (1,1)$) are blocked into sets of four consecutive
columns ($\vs = (1,4)$), and $k\!=\!2$ blocks are retained per scope.

\textbf{4:8 sparsity.}
The view is again the physical layout. Each block spans two adjacent
columns ($\vb = (1,2)$), and scopes contain four blocks ($\vs = (1,4)$)
with $k\!=\!2$, keeping 4 of every 8 columns at block granularity.

\textbf{Coupled 2:4.}
A strided view pairs columns that are 8 apart within each
16-column segment:
$\View = (M, K/16, 8, 2):(K, 16, 1, 8)$, so that element
$[m, g, i, j]$ maps to $W_{m,\; 16g + i + 8j}$.
Blocks of size $\vb = (1,1,1,2)$ each contain a column pair, and blocks
$\vs = (1,1,4,1)$ collect four such pairs; retaining $k\!=\!2$ enforces
2:4 sparsity over coupled columns. The coupling amortizes mask metadata
over pairs, improving the compression ratio from $9/16$ to $17/32$ for
half-precision weights (\Cref{app:coupled_24_compression}).

\textbf{16-column block sparsity.}
A strided view couples rows that are 8 apart within 16-row chunks:
$\View = (8, 2, K):(K, 8K, 1)$, mapping $[p, r, c]$ to $W_{p+8r,\; c}$.
Blocks of 16 contiguous columns ($\vb = (1,1,16)$) are blocked in
row-pairs ($\vs = (1,2,1)$) with $k\!=\!1$, yielding 50\% block sparsity
where paired rows share the same column mask. This pattern is
hardware-friendly because it preserves contiguous memory access along the
GEMM reduction dimension (\Cref{app:block_sparse_gemm}).

%% file: 5_integrations.tex
\sthree{} provides the structural foundation for pruning algorithms. We derive the key
operations for Optimal Brain Damage (OBD)~\citep{lecun1989optimal} and Optimal Brain Surgeon
(OBS)~\citep{hassibi1993optimal}. While recent methods such as
SparseGPT~\citep{frantar2023sparsegpt} and Wanda~\citep{sun2023wanda} implement efficient
approximations of these criteria for large models, their formulations are tied to specific
sparsity patterns. We return to the foundational OBD/OBS framework and derive structured
extensions (S-OBD, S-OBS) that operate on arbitrary \sthree{} specifications.

\subsection{Block-Level Quantities}

Since $\Elements(j)$ enumerates the indices of block $j$ (see \Cref{appendix:blockspec} for the
formal definition), any per-element quantity lifts to blocks by restricting its indices to
$\Elements(j)$. In particular, the block gradient is $\nabla_j L = (\partial L / \partial w_e)_{e
      \in \Elements(j)}$, the block Hessian is $\vH_j = (\partial^2 L / \partial w_{e_1} \partial
  w_{e_2})_{e_1, e_2 \in \Elements(j)}$, and the cross-block Hessian is $\vH_{j_1, j_2} = (\partial^2
  L / \partial w_{e_1} \partial w_{e_2})_{e_1 \in \Elements(j_1), e_2 \in \Elements(j_2)}$. The same
principle applies to coupled blocks and blocks via their respective element sets.

\subsection{Optimal Brain Damage (OBD)}
We extend OBD saliencies based on diagonal Hessian approximations to block-structured sparsity.

\begin{definition}[Block OBD Saliency]
  \label{def:block_obd_saliency}
  Under a diagonal Hessian approximation, the saliency of block~$j$ is:
  \begin{equation}
    S_j^{\text{OBD}} = \frac{1}{2} \sum_{e \in \Elements(j)} H_{ee} w_e^2
  \end{equation}
\end{definition}

At the scope level, pruning reduces to $k$-sparsity: within each scope $\ell$, retain the $k$
blocks with largest saliency $S_j^{\text{OBD}}$ and prune the rest.

\subsection{Optimal Brain Surgeon (OBS)}

OBS uses the full inverse Hessian for optimal weight updates.

\begin{definition}[Block OBS Saliency]
  \label{def:block_obs_saliency}
  For block $j$ with weight vector $\vw_j$:
  \begin{equation}
    S_j^{\text{OBS}} = \frac{1}{2} \vw_j^T \left( [\vH^{-1}]_{jj} \right)^{-1} \vw_j
  \end{equation}
  where $[\vH^{-1}]_{jj}$ is the diagonal block of the inverse Hessian.
\end{definition}

Saliencies are aggregated across tensors for coupled specifications: $S_\ell^{\text{coup}} =
  \sum_{i=1}^{K} S_\ell^{(i)}$.

\begin{theorem}[Optimal Weight Update for Block Pruning]
  \label{thm:obs_update}
  When pruning block $j$, the optimal update to remaining weights is:
  \begin{equation}
    \delta \vw = -\vH^{-1}_{:,I_j} \left( [\vH^{-1}]_{jj} \right)^{-1} \vw_j
  \end{equation}
  where $I_j = \Elements(j)$ and $\vH^{-1}_{:,I_j}$ are the columns of $\vH^{-1}$ corresponding to block $j$ elements.
\end{theorem}

\subsection{Efficient Block OBS via Schur Complement}

For iterative pruning, we maintain $\vH^{-1}$ efficiently. After pruning block $j$, the updated
inverse Hessian follows from the Schur complement:

\begin{theorem}[Inverse Hessian Update]
  \label{thm:schur_update}
  After pruning block $j$ with index set $I_j$:
  \begin{equation}
    \vH'^{-1} = \vH^{-1} - \vH^{-1}_{:,I_j} \left( [\vH^{-1}]_{jj} \right)^{-1} \vH^{-1}_{I_j,:}
  \end{equation}
\end{theorem}

\begin{algorithm}[t]
  \caption{Structured OBS Pruning via \sthree}
  \label{alg:scope_obs}
  \begin{algorithmic}[1]
    \Require Weight tensor $\vW$, Hessian inverse $\vH^{-1}$, \sthree{} spec $\mathcal{S}$, blocks to keep $k$
    \Ensure Pruned weights $\vW'$, mask $\mathbf{M}$
    \State $\mathbf{M} \gets \mathbf{1}$ \Comment{Initialize mask}
    \For{each scope $\ell$ in $\mathcal{S}$}
    \For{each block $j \in \Blocks(\ell)$}
    \State $\vw_j \gets \Elements_\mathcal{S}(j)$
    \State $S_j \gets \frac{1}{2} \vw_j^T ([\vH^{-1}]_{jj})^{-1} \vw_j$
    \EndFor
    \State $\mathcal{P} \gets \Blocks(\ell) \setminus \text{top-}k(S_j)$ \Comment{Blocks to prune}
    \For{$j \in \mathcal{P}$ in order of increasing $S_j$}
    \State $\delta \vw \gets -\vH^{-1}_{:,I_{j}} ([\vH^{-1}]_{jj})^{-1} \vw_{j}$
    \State $\vW \gets \vW + \delta \vw$ \Comment{Optimal weight update}
    \State $M_{j} \gets 0$ \Comment{Prune block}
    \State Update $\vH^{-1}$ via Theorem~\ref{thm:schur_update}
    \EndFor
    \EndFor
    \State \Return $\vW \odot \mathbf{M}$, $\mathbf{M}$
  \end{algorithmic}
\end{algorithm}

\subsection{SparseGPT and Wanda as heuristics}

SparseGPT~\citep{frantar2023sparsegpt} uses row-wise OBS with the empirical Hessian $\vH =
  \frac{1}{N}\mathbf{X}^T\mathbf{X}$, where $\mathbf{X} \in \R^{N \times K}$ is the calibration
input. The block extension applies OBS independently per row-block intersection, using
block-diagonal Hessian approximations. Wanda~\citep{sun2023wanda} is a special case that uses
magnitude$\times$activation scoring without compensation (\Cref{appendix:wanda}).
\subsection{Theoretical Analysis} \label{sec:theory}

We now establish formal properties of \sthree{}: what patterns it can express, the computational
cost of block-structured pruning, and what sparsity ratios are achievable under the scope
constraint.

\paragraph{Expressiveness}

We show that \sthree{} is universal over axis-aligned patterns and complete for N:M sparsity.

\begin{theorem}[Universality]
  \label{thm:universality}
  \sthree{} can express any axis-aligned hyperrectangular sparsity pattern over a tensor.
\end{theorem}

\begin{proof}
  Any axis-aligned hyperrectangle in tensor coordinates can be mapped to a block via an appropriate view. The view reshapes the tensor such that the hyperrectangle aligns with a contiguous block. Since views are arbitrary (subject to size preservation), any such pattern is expressible.
\end{proof}

\begin{theorem}[N:M Completeness]
  \label{thm:nm_completeness}
  For any $N < M$, the N:M sparsity pattern is expressible in \sthree{}.
\end{theorem}

\begin{proof}
  Set $\View$ to reshape the tensor with innermost dimension $M$. Set $\vb = (1, \ldots, 1)$ (scalar blocks). Set $\vs$ to block $M$ consecutive elements. The scope constraint enforces exactly $N$ non-zero per $M$ elements.
\end{proof}

The scope constraint discretizes the set of achievable sparsity ratios.

\begin{theorem}[Achievable Sparsity]
  \label{thm:achievable_sparsity}
  For \sthree{} specification with block size $|\Scope|_B$ blocks:
  \begin{equation}
    \rho \in \left\{ \frac{k}{|\Scope|_B} : k \in \{0, 1, \ldots, |\Scope|_B\} \right\}
  \end{equation}
  where $\rho$ is the achievable block sparsity ratio within each scope.
\end{theorem}

\paragraph{Complexity Analysis}

The block structure determines the cost of OBS-based pruning.

\begin{theorem}[Block OBS Complexity]
  \label{thm:obs_complexity}
  Block OBS with $B$ blocks of size $b$ requires $O(B \cdot b^3)$ saliency computation through block Hessian inverse, $O(B^2 \cdot b^2)$ weight updates, and  $O(k \cdot B \cdot b^3)$ pruning iterations with Schur updates for $k$ blocks.
\end{theorem}

%% file: 6_experiments.tex
We validate \sthree{} by implementing the Structured OBS (S-OBS) entirely through the framework's
View, Block, and Scope abstractions. All sparsity patterns below are specified as \sthree{}
configurations and pruned by the same algorithm (Algorithm~\ref{alg:scope_obs}).

We focus on zero-shot (post-training) pruning, where a pre-trained model is sparsified in one pass
using only a small calibration set, without any retraining or fine-tuning. This setting is the most
practical for large language models, where retraining is prohibitively expensive. Following the
experimental protocol of Wanda~\citep{sun2023wanda}, we use 1024 calibration samples of 1024 tokens
from the English scope of the Colossal Clean Crawled Corpus (C4)~\citep{c4dataset} and evaluate
perplexity on WikiText-2~\citep{wikitextdataset}.

\paragraph{Setup} We use SparseGPT~\citep{frantar2023sparsegpt} as our baseline and report the relative output Frobenius norm ${\normf{\mathbf{X}(\hat{\vW}\!-\!\vW)^T}}/{\normf{\mathbf{X}\vW^T}}$\,, where $\hat{\vW}$ is the pruned weight matrix. All experiments run on a single NVIDIA A100 (40GB), with S-OBS implemented via \sthree{} (per-row $\vH^{-1}$ with Schur updates).
  
\begin{table}[t]
  \centering
  \small
  \begin{tabular}{@{}llccc@{}}
    \toprule
    Model                       & Method    & PPL ($\downarrow$) & PPL Increase & Time \\
    \midrule
    \multirow{4}{*}{Qwen3-0.6B} & Dense     & 26.13              & —            & —    \\
                                & S-OBD     & 369.78             & $+$1315.1\%  & 79s  \\
                                & SparseGPT & \textbf{126.23}    & $+$383.0\%   & 165s \\
                                & S-OBS     & 154.16             & $+$489.9\%   & 223s \\
    \midrule
    \multirow{4}{*}{Qwen3-1.7B} & Dense     & 21.05              & —            & —    \\
                                & S-OBD     & 125.75             & $+$497.4\%   & 167s \\
                                & SparseGPT & 56.36              & $+$167.7\%   & 352s \\
                                & S-OBS     & \textbf{46.26}     & $+$119.8\%   & 963s \\
    \midrule
    \multirow{3}{*}{Qwen3-4B}   & Dense     & 16.44              & —            & —    \\
                                & SparseGPT & 30.11              & $+$83.1\%    & 435s \\
                                & S-OBS     & \textbf{27.93}     & $+$69.9\%    & 4206s \\
    \bottomrule
  \end{tabular}
  \caption{End-to-end 2:4 pruning of Qwen3 models.
    WikiText-2 word perplexity, 1024 C4 calibration samples. S-OBS achieves lower
    perplexity on Qwen3-1.7B and Qwen3-4B but not on Qwen3-0.6B, despite lower
    per-linear errors on all three (\Cref{sec:appendix_0.6b,sec:appendix_linear_type}).}
  \label{tab:e2e}
\end{table}

We run a toy experiment where we extract the first linear layer from Qwen3-4B~\citep{qwen3}: weight
matrix $\vW \in \R^{2560 \times 9728}$ and calibration inputs $\mathbf{X} \in \R^{244449 \times
    9728}$.

\Cref{tab:results} summarizes results across four sparsity configurations, each expressed as a
different \sthree{} specification (see \Cref{app:patterns} for the full specifications) but pruned
by the same S-OBS implementation. All four patterns use the same pruning algorithm (Algorithm~\ref{alg:scope_obs}) with different
\sthree{} specifications. No pattern-specific code is required.

S-OBS consistently outperforms SparseGPT by 16--20\%. The improvement is remarkably stable across
all four patterns, from fine-grained 2:4 to coarse 16-column block sparsity, demonstrating that
per-row Hessian compensation generalizes across structures.

The quality gap comes from per-row compensation, not mask selection. Replacing SparseGPT's diagonal
mask selection ($w^2/d^2$) with exact $\binom{4}{2}$ enumeration yields no improvement. The gap
arises entirely from maintaining a per-row $\vH^{-1}$ with Schur updates, versus SparseGPT's shared
column-sequential approximation.

\subsection{End-to-End LLM Pruning}
\label{sec:e2e}

We evaluate S-OBS on full model pruning to assess whether per-layer output error improvements
translate into end-to-end perplexity gains.

We prune Qwen3-1.7B layer-by-layer with 2:4 sparsity
($\vb\!=\!(1,1),\;\vs\!=\!(1,4)$), using 1024 C4 calibration samples of length 1024. Each decoder
layer is processed independently: we capture inputs to all seven linear projections (Q, K, V, O,
Gate, Up, Down) via a single forward pass, prune each projection, then propagate through the pruned
layer to obtain inputs for the next. Perplexity is evaluated on WikiText-2 after every fourth layer
and at the final layer.

\begin{figure*}[ht]
  \centering
  \begin{subfigure}[t]{0.35\linewidth}
    \centering
    \includegraphics[width=\linewidth]{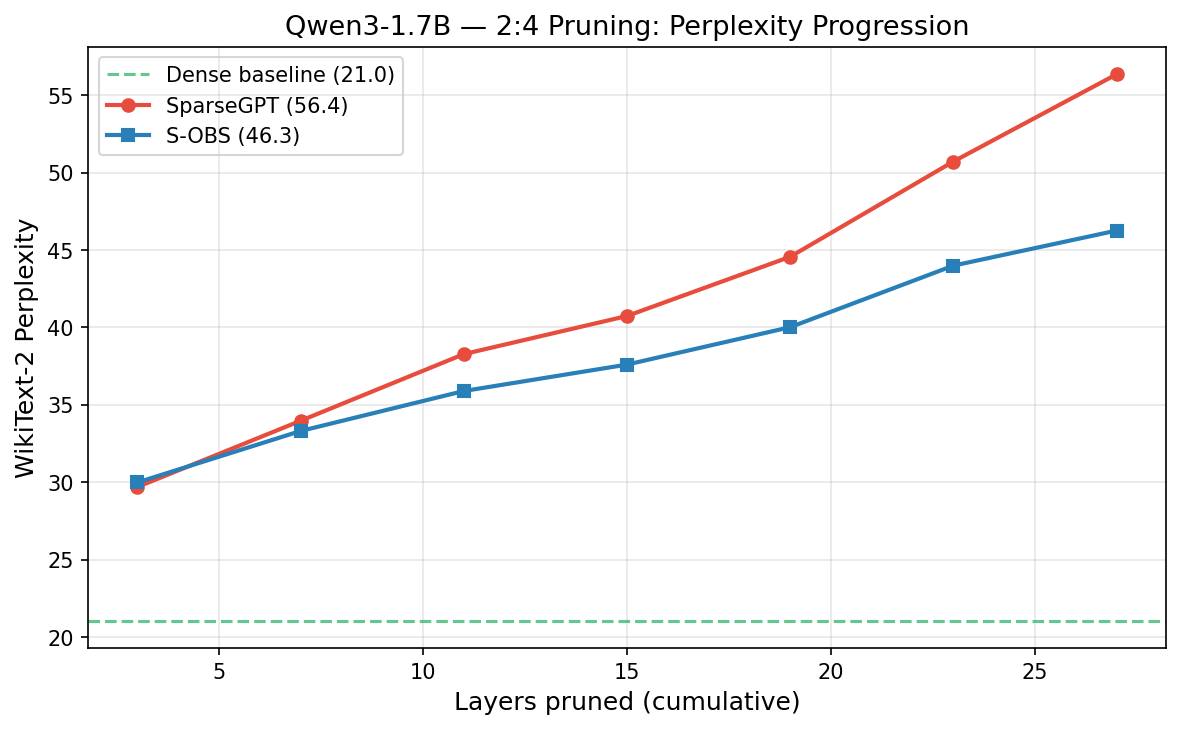}
    \caption{Perplexity progression}
    \label{fig:e2e_ppl}
  \end{subfigure}\hfill
  \begin{subfigure}[t]{0.63\linewidth}
    \centering
    \includegraphics[width=\linewidth]{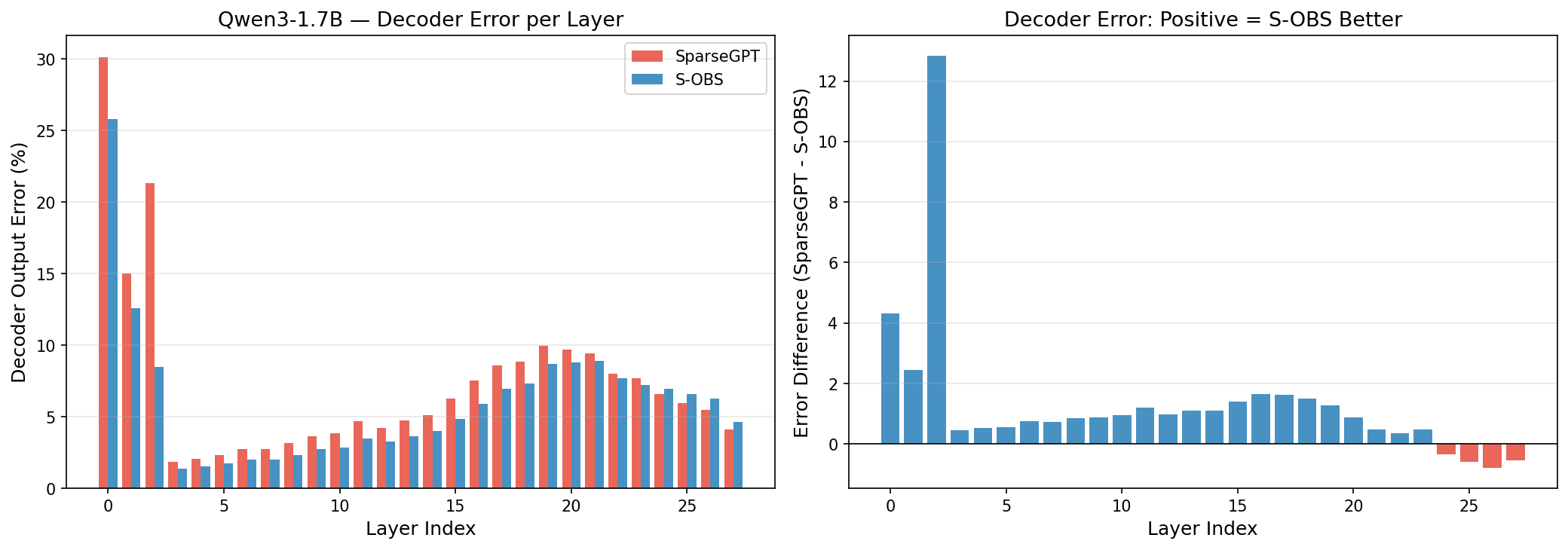}
    \caption{Decoder output error}
    \label{fig:e2e_decoder}
  \end{subfigure}
  \caption{End-to-end 2:4 pruning of Qwen3-1.7B.
    \Cref{fig:e2e_ppl}:~WikiText-2 perplexity after pruning every fourth decoder layer.
    S-OBS maintains a growing advantage throughout, reaching 46.3 vs.\ 56.4 for SparseGPT.
    \Cref{fig:e2e_decoder}:~Per-layer decoder output error. S-OBS achieves lower error
    in 25 of 28 layers; SparseGPT is slightly better only in the final three.}
  \label{fig:e2e_ppl_decoder}
\end{figure*}

\begin{table}[t]
  \centering
  \small
  \begin{tabular}{@{}lcccc@{}}
    \toprule
    Config       & S-OBS   & SparseGPT & Impr.     & Time \\
    \midrule
    2:4          & 11.87\% & 14.12\%   & $+$16.0\% & 53s  \\
    Coupled 2:4  & 15.75\% & 19.01\%   & $+$17.1\% & 217s \\
    4:8          & 15.92\% & 19.00\%   & $+$16.2\% & 279s \\
    16-col block & 26.91\% & 33.46\%   & $+$19.6\% & 108s \\
    \bottomrule
  \end{tabular}
  \caption{Structured OBS via \sthree{} vs.\ SparseGPT on the first decoder layer of
    Qwen3-4B across four sparsity configurations (50\% sparsity). S-OBS reports the best
    variant (True OBS with per-row Schur updates). ``Impr.'' is relative loss reduction
    $(1 - \text{S-OBS}/\text{SparseGPT}) \times 100$.}
  \label{tab:results}
\end{table}

\begin{figure*}[t]
  \centering
  \begin{subfigure}[t]{0.48\linewidth}
    \centering
    \includegraphics[width=\linewidth]{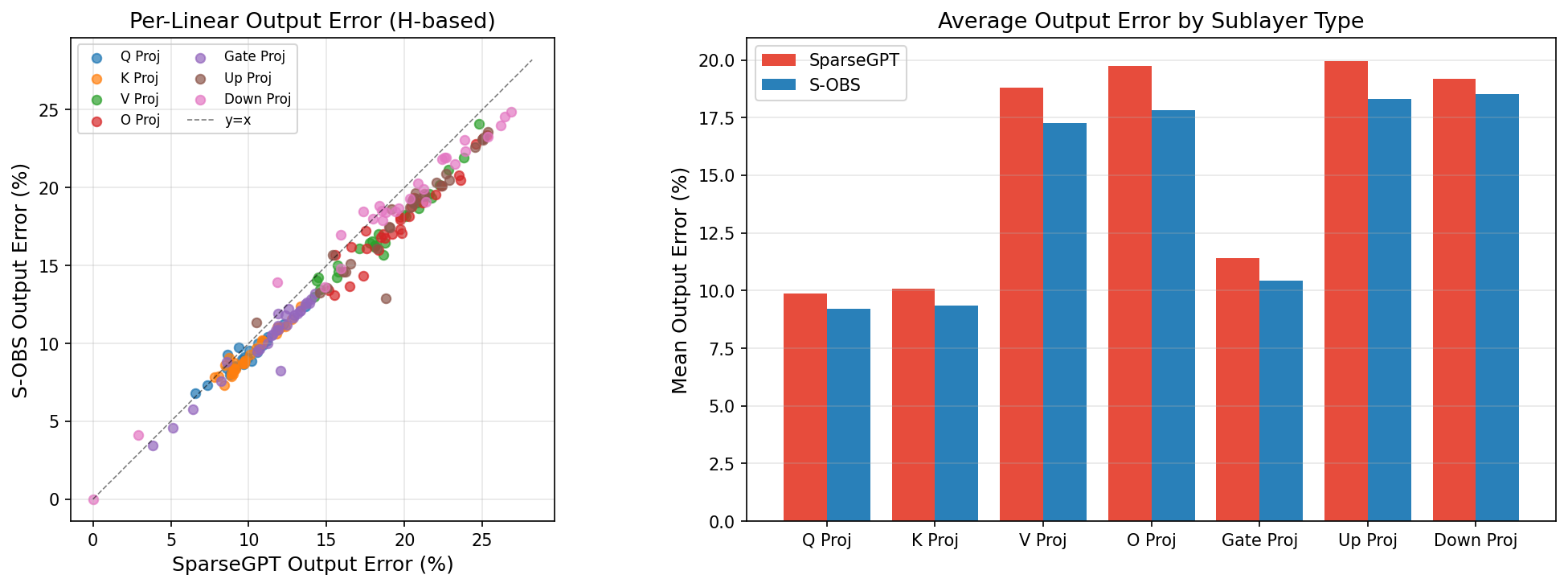}
    \caption{Per-linear output error}
    \label{fig:e2e_linear}
  \end{subfigure}\hfill
  \begin{subfigure}[t]{0.48\linewidth}
    \centering
    \includegraphics[width=\linewidth]{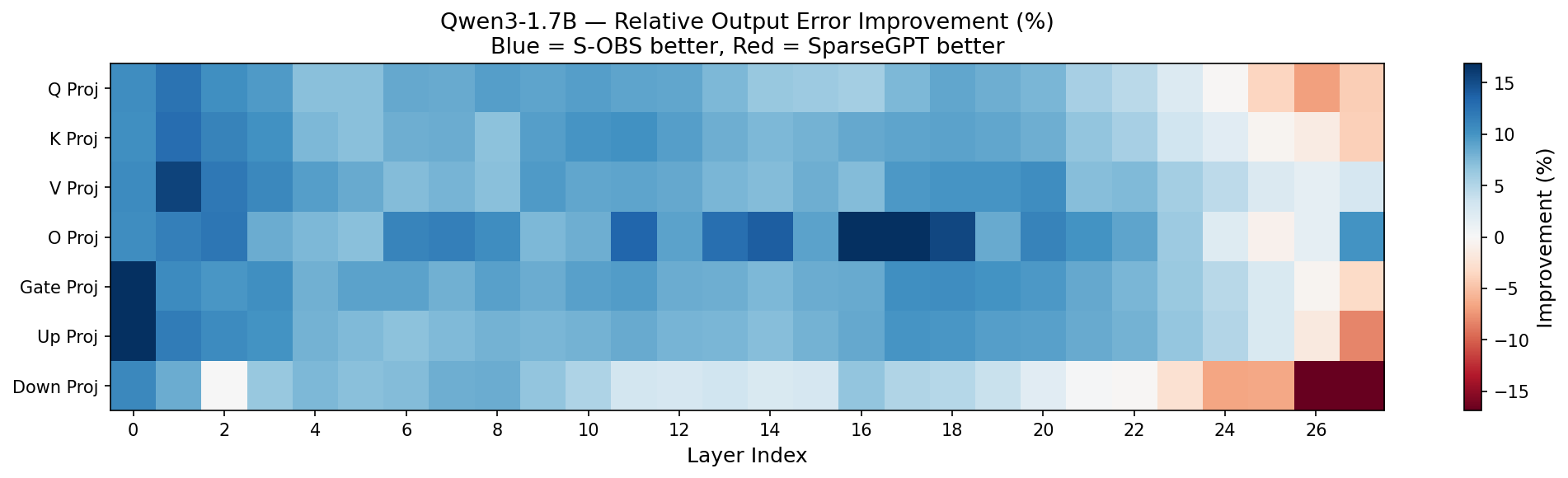}
    \caption{Improvement heatmap}
    \label{fig:e2e_heatmap}
  \end{subfigure}
  \caption{Per-linear analysis of Qwen3-1.7B 2:4 pruning.
    \Cref{fig:e2e_linear}:~Output error scatter (H-based metric). Points below the
    diagonal indicate S-OBS achieves lower error. S-OBS wins on 177 of 196 linear layers.
    \Cref{fig:e2e_heatmap}:~Relative improvement heatmap across layers and sublayer types.
    Blue indicates S-OBS is better. S-OBS dominates except for a few MLP projections in
    the final layers.}
  \label{fig:e2e_linear_heatmap}
\end{figure*}
\Cref{tab:e2e} and \Cref{fig:e2e_ppl} show that S-OBS achieves
17.9\% lower perplexity degradation than SparseGPT (46.26 vs.\ 56.36), with the gap
widening steadily as more layers are pruned.

S-OBS consistently achieves lower per-layer output error. \Cref{fig:e2e_linear} shows that S-OBS
achieves lower H-based output error on 177 of 196 individual linear projections. The heatmap
(\Cref{fig:e2e_heatmap}) confirms this advantage spans all sublayer types and nearly all layers.

Per-linear gains translate to model-level gains on Qwen3-1.7B. The per-layer decoder output error
(\Cref{fig:e2e_decoder}) shows S-OBS is better in 25 of 28 layers, and this advantage compounds
into a substantial perplexity improvement.

SparseGPT overtakes S-OBS in the final 3 layers because the Hessian becomes increasingly ill-conditioned in later layers as pruning errors accumulate in the propagated activations. The effect is most pronounced for
projections with large input dimension ($K \gg n_{\text{cal}}$, \eg the \texttt{down\_proj} with
$K\!=\!6144$ vs.\ $n_{\text{cal}}\!=\!1024$), where the rank-deficient Hessian inverse used by
S-OBS becomes less reliable. SparseGPT's column-sequential updates are more robust in this regime.

The per-linear output error advantage of S-OBS is consistent across model sizes; however, it does
not always translate to better perplexity. On Qwen3-0.6B, S-OBS wins on 178 of 196 linear layers
yet achieves higher final perplexity (154.2 vs.\ 126.2) because the late-layer decoder error
reversal is more severe in smaller models. Notably, decoder error varies substantially
across layers; middle layers appear more prunable than early or late ones, though
this observation is conditioned on the C4 calibration data and may not generalize across
domains. We present the full Qwen3-0.6B results in \Cref{sec:appendix_0.6b} and
per-projection error with further discussion in
\Cref{sec:appendix_linear_type} (\Cref{fig:error_progression,fig:error_progression_4b}).

\paragraph{S-OBS runtime.}
S-OBS maintains a per-row $K \times K$ inverse Hessian, making its cost
$O(M \cdot K^2)$ per linear projection. On Qwen3-4B ($d\!=\!2560$,
$d_{\text{ff}}\!=\!9728$), the \texttt{down\_proj} ($K\!=\!9728$) alone
accounts for 73\% of the total pruning time because $K$ is
$3.8\times$ larger than for attention projections ($K\!=\!2560$), and
the per-row Schur update scales quadratically. SparseGPT avoids this
bottleneck via column-sequential updates that are $O(M \cdot K)$ per
column, at the cost of lower pruning quality.

%% file: 7_conc.tex
We have presented \sthree{} (Structured Sparsity Specification), a unified algebraic framework for
neural network pruning. By decomposing sparsity into a composed View and two shape-based
specifications (Block and Scope) with coupling across tensors, \sthree{} provides a principled
approach to specifying diverse sparsity patterns under a common formalism.

Using \sthree{}, we implemented Structured OBS (S-OBS) with per-row Hessian inverse maintenance via
Schur complement updates. On a single layer, S-OBS surpasses SparseGPT by 16--20\% across four
distinct sparsity configurations, all pruned by the same algorithm with different \sthree{}
specifications. End-to-end pruning of Qwen3-1.7B with 2:4 sparsity yields 17.9\% lower perplexity
degradation than SparseGPT (46.3 vs.\ 56.4). Crucially, our experiments show that the quality gap
comes from per-row Hessian compensation, not mask selection: replacing SparseGPT's diagonal scoring
with exact subset enumeration yields no improvement, while S-OBD (no compensation) performs far
worse than both methods.

The framework has limitations. S-OBS is 2--3$\times$ slower than SparseGPT due to per-row Schur
updates, though this is a one-time cost amortized over the entire deployment lifetime of the pruned
model. The more fundamental limitation is that the per-row inverse Hessian becomes rank-deficient
when calibration samples are fewer than the input dimension, causing late-layer degradation that
prevents end-to-end gains on smaller models. Beyond post-training pruning, \sthree{} can also serve as a structural backbone for dynamic sparse training (\Cref{appendix:dynamic_sparse}).

%% file: appendix.tex
\definecolor{codebg}{HTML}{F6F8FA}
\definecolor{codeframe}{HTML}{D0D7DE}
\definecolor{kwcolor}{HTML}{0033B3}
\definecolor{cmtcolor}{HTML}{6E7781}
\definecolor{strcolor}{HTML}{A31515}
\definecolor{numcolor}{HTML}{116644}

\lstdefinestyle{pythonstyle}{
  language=Python,
  basicstyle=\ttfamily\small,
  keywordstyle=\color{kwcolor}\bfseries,
  commentstyle=\color{cmtcolor}\itshape,
  stringstyle=\color{strcolor},
  numberstyle=\tiny\color{cmtcolor},
  identifierstyle=\color{black},
  morekeywords={True, False, None, self, as},
  backgroundcolor=\color{codebg},
  frame=single,
  framerule=0.5pt,
  rulecolor=\color{codeframe},
  columns=fullflexible,
  keepspaces=true,
  breaklines=true,
  showstringspaces=false,
  tabsize=4,
  xleftmargin=6pt,
  xrightmargin=6pt,
  aboveskip=6pt,
  belowskip=6pt,
}

\clearpage
\begin{center}
  {\LARGE\bfseries Appendix}
\end{center}
\vspace{1em}

\section{Notation}
Throughout this document, we will use the following notations:

\begin{table}[H]
  \centering
  \small
  \begin{tabular}{@{}cl@{}}
    \toprule
    Symbol                         & Description               \\
    \midrule
    $\T$                           & A dense tensor            \\
    $\vs = (s_0, \ldots, s_{n-1})$ & Shape tuple (extents)     \\
    $\vd = (d_0, \ldots, d_{n-1})$ & Stride tuple              \\
    $\Layout = \vs:\vd$            & Layout specification      \\
    $|\Layout|$                    & Size of a layout          \\
    $\cosize(\Layout)$             & Cosize (memory footprint) \\
    $\mathcal{D}$                  & Domain specification      \\
    $\mathbf{o}_D, \mathbf{e}_D$   & Domain offset and extent  \\
    $\B$                           & Block specification       \\
    $\Scope$                           & Scope specification       \\
    $\otimes$                      & Coupling operator         \\
    \bottomrule
  \end{tabular}
  \caption{Notation summary.}
  \label{tab:notation}
\end{table}

\section{Domain Specification}
\label[appendix]{app:domain}

The Domain specifies which sub-tensor is subject to sparsification. Elements outside the
domain are untouched by pruning.

\begin{definition}[Domain]
 \label{def:domain}
 A domain $\mathcal{D}$ over tensor $\T$ with physical layout $\Layout_{\text{phys}}$ is defined by:
 \begin{itemize}
  \item \textbf{Offset} $\mathbf{o}_D = (o_0, \ldots, o_{n-1})$: Starting coordinates in the physical tensor
  \item \textbf{Extent} $\mathbf{e}_D = (e_0, \ldots, e_{n-1})$: Size of the domain along each dimension
 \end{itemize}
 satisfying $o_k + e_k \leq s_k^{(\Layout_{\text{phys}})}$ for all $k$.
\end{definition}

\begin{definition}[Domain Size]
 \label{def:domain_size}
 The number of elements in the domain:
 \begin{equation}
  |\mathcal{D}| = \prod_{k=0}^{n-1} e_k
 \end{equation}
\end{definition}

\begin{definition}[Domain Layout]
 \label{def:domain_layout}
 The domain induces a layout $\Layout_D$ with:
 \begin{equation}
  \Layout_D = \mathbf{e}_D : \vd^{(\Layout_{\text{phys}})}
 \end{equation}
 The domain maps local coordinates $\vi_D \in \dom(\Layout_D)$ to physical coordinates:
 \begin{equation}
  \psi_D(\vi_D) = \mathbf{o}_D + \vi_D
 \end{equation}
\end{definition}

\begin{definition}[Domain Element Set]
 \label{def:domain_elements}
 The set of physical linear indices covered by domain $\mathcal{D}$:
 \begin{equation}
  \Elements(\mathcal{D}) = \left\{ \phi_{\Layout_{\text{phys}}}(\mathbf{o}_D + \vi_D) : \vi_D \in \dom(\Layout_D) \right\}
 \end{equation}
\end{definition}

\begin{example}[Sub-matrix Domain]
 For a matrix $\vW \in \R^{128 \times 128}$ with row-major storage, to sparsify only a $16 \times 16$ block starting at $(32, 64)$:
 \begin{equation}
  \mathbf{o}_D = (32, 64), \quad \mathbf{e}_D = (16, 16)
 \end{equation}
 The domain covers $16 \times 16 = 256$ elements. The remaining $128^2 - 256 = 16128$ elements are outside the domain and will not be pruned.
\end{example}

\subsection{Generalized Domain via Layout}

For non-contiguous or strided sub-tensor selection, we generalize the domain using a
layout.

\begin{definition}[Generalized Domain]
 \label{def:generalized_domain}
 A generalized domain $\mathcal{D}_G$ is specified by:
 \begin{itemize}
  \item \textbf{Domain Layout} $\Layout_D$: Defines the logical shape of the domain
  \item \textbf{Embedding} $\phi_D: \dom(\Layout_D) \to [0, |\Layout_{\text{phys}}|)$: Maps domain coordinates to physical linear indices
 \end{itemize}
\end{definition}

\begin{example}[Strided Domain]
 To select every other row of a $128 \times 128$ matrix (rows $0, 2, 4, \ldots, 126$):
 \begin{equation}
  \Layout_D = (64, 128):(256, 1)
 \end{equation}
 This creates a domain of shape $64 \times 128$ where the row stride is $256 = 2 \times 128$, effectively selecting alternate rows.
\end{example}

\begin{example}[Diagonal Blocks Domain]
 To select $k$ diagonal $b \times b$ blocks from a $kb \times kb$ matrix:
 \begin{equation}
  \Layout_D = (k, b, b):((b+1) \cdot kb \cdot b, kb, 1)
 \end{equation}
 This selects the diagonal blocks while skipping off-diagonal regions.
\end{example}

\subsection{Domain Composition}

Multiple domains can be composed to create complex selection patterns.

\begin{definition}[Domain Union]
 \label{def:domain_union}
 For domains $\mathcal{D}_1, \mathcal{D}_2$ over the same tensor:
 \begin{equation}
  \Elements(\mathcal{D}_1 \cup \mathcal{D}_2) = \Elements(\mathcal{D}_1) \cup \Elements(\mathcal{D}_2)
 \end{equation}
\end{definition}

\begin{definition}[Domain Complement]
 \label{def:domain_complement}
 The complement domain $\bar{\mathcal{D}}$ contains all elements not in $\mathcal{D}$:
 \begin{equation}
  \Elements(\bar{\mathcal{D}}) = [0, |\T|) \setminus \Elements(\mathcal{D})
 \end{equation}
\end{definition}

\subsection{Constraint: Block View over Domain}

With domains, the Block View constraint is relaxed:

\begin{definition}[Block View over Domain]
 \label{def:block_view_domain}
 Given domain $\mathcal{D}$ with layout $\Layout_D$, the Block View $\View_B$ satisfies:
 \begin{equation}
  |\View_B| = |\mathcal{D}|
 \end{equation}
 The Block View reshapes the \emph{domain}, not the full tensor.
\end{definition}

\section{Block Specification Details}\label[appendix]{appendix:blockspec}

\subsection{Block Index Mapping}

\begin{theorem}[Block-to-Element Mapping]
  \label{thm:block_to_element}
  For block index $\vj \in \Z^n$ and intra-block offset $\vo \in \Z^n$ with $0 \leq o_k < b_k$:
  \begin{equation}
    \idx(\vj, \vo) = \sum_{k=0}^{n-1} (j_k \cdot b_k + o_k) \cdot d_k^{(\View)}
  \end{equation}
\end{theorem}

\begin{proof}
  The view coordinate corresponding to block $\vj$ at offset $\vo$ is:
  $\mathbf{v} = (j_0 \cdot b_0 + o_0, \ldots, j_{n-1} \cdot b_{n-1} + o_{n-1})$.
  Applying the view layout function:
  $\phi_{\View}(\mathbf{v}) = \sum_{k=0}^{n-1} v_k \cdot d_k^{(\View)} = \sum_{k=0}^{n-1} (j_k \cdot b_k + o_k) \cdot d_k^{(\View)}$.
\end{proof}

\begin{definition}[Block Element Enumeration]
  \label{def:block_enumeration}
  For block $\vj$, the set of linear indices is:
  \begin{equation}
    \Elements(\vj) = \left\{ \idx(\vj, \vo) : \vo \in \prod_{k=0}^{n-1} [0, b_k) \right\}
  \end{equation}
\end{definition}

\begin{definition}[Element-to-Block Mapping]
  \label{def:element_to_block}
  The function $\beta: [0, |\T|) \to [0, |\Scope|_B)$ maps element index to block index:
  \begin{equation}
    \beta(e) = \sum_{k=0}^{n-1} \left\lfloor \frac{v_k(e)}{b_k} \right\rfloor \cdot \prod_{\ell < k} g_\ell^{(B)}
  \end{equation}
  where $v_k(e)$ is the $k$-th view coordinate of element $e$.
\end{definition}

\section{Canonical Sparsity Patterns}
\label[appendix]{app:patterns}

We demonstrate \sthree{}'s expressiveness by encoding canonical pruning patterns.
Each pattern is fully determined by the View $\View$, Block Shape $\vb$, and Scope
Shape $\vs$; \Cref{tab:patterns} summarizes five representative patterns.
The four experimental patterns from \Cref{sec:exp_patterns} are visualized in
\Cref{fig:patterns_app}.

\begin{figure}[!t]
  \centering
  \includegraphics[width=0.9\linewidth]{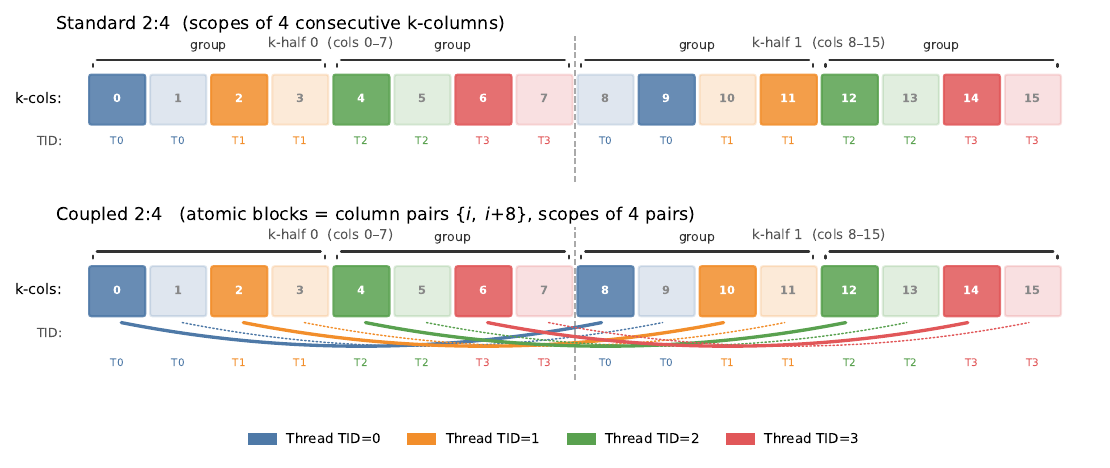}
  \caption{Standard 2:4 (top) vs.\ Coupled 2:4 (bottom) over a single
    A-fragment row of \texttt{mma.m16n8k16} (16 k-columns).  Cell color
    encodes the owning TID; faded cells are pruned.  Arcs link the
    column pair $\{c,\; c\!+\!8\}$; solid arcs are retained pairs, dashed arcs
    are pruned pairs.  Standard 2:4 requires independent masks for each k-half.
    Coupled 2:4 uses one mask for both halves, halving the metadata and
    improving the compression ratio from $9/16$ to $17/32$
    (\Cref{app:coupled_24_compression}).}
  \label{fig:coupled_vs_standard}
\end{figure}

\subsection{Unstructured Sparsity}

Unstructured sparsity treats every scalar as an independent block. The view is the
identity ($\View = \Layout_{\text{phys}}$), block shape is all-ones
($\vb = (1, 1, \ldots, 1)$), and a single block spans the entire tensor
($\vs = (|\T|,)$). Retaining the top-$k$ blocks by saliency recovers standard
magnitude or OBS pruning at arbitrary sparsity ratios.

\subsection{N:M Sparsity}

For 2:4 sparsity on a matrix $\vW \in \R^{M \times K}$, the view is the physical
layout and scopes block columns into consecutive fours:
\begin{equation}
  \View = (M, K):(K, 1), \quad \vb = (1, 1), \quad \vs = (1, 4)
\end{equation}
Each scope contains four scalar blocks along the column dimension, and retaining
$k\!=\!2$ per scope enforces exactly 2-of-4 sparsity. Generalizing to $N\!:\!M$ only
requires changing the innermost block dimension to $M$ and setting $k\!=\!N$.

\subsection{Block Sparsity}

For $b \times b$ block sparsity the view is again the identity, but the block shape
absorbs the spatial extent: $\vb = (b, b)$. The scope shape controls how blocks
compete: setting $\vs$ equal to the number of block rows yields row-wise sparsity,
while a global block enforces a tensor-wide budget.

\subsection{Channel Pruning}

For a convolutional layer $\vW \in \R^{C_{\text{out}} \times C_{\text{in}} \times H
\times W}$, the view flattens spatial and input dimensions so that each output channel
becomes a single block:
\begin{equation}
  \View = (C_{\text{out}}, C_{\text{in}} \!\cdot\! H \!\cdot\! W):(C_{\text{in}} \!\cdot\! H \!\cdot\! W, 1), \quad
  \vb = (1, C_{\text{in}} \!\cdot\! H \!\cdot\! W)
\end{equation}
A global block ($\vs = (C_{\text{out}},)$) then selects channels by saliency.

\subsection{Attention Head Pruning}

For multi-head attention with $h$ heads and dimension $d$, Q, K, and V projections use
block shape $(1, d/h, d)$ along the head dimension, while the output projection
uses $(d, 1, d/h)$ to match. Coupling all four matrices along the head dimension
ensures that pruning a head simultaneously removes the corresponding rows and columns
across Q, K, V, and O.

\subsection{Partial Tensor Sparsification}

\sthree{} can sparsify only a sub-region of a tensor via the Domain specification
(\Cref{app:domain}). As an example, consider a sparse MLP with a dense residual
pathway: for $\vW \in \R^{d \times d}$, the first $d/4$ output rows are kept dense
while the remaining $3d/4$ rows receive 2:4 sparsity. The domain offsets
$\mathbf{o}_D = (d/4, 0)$, $\mathbf{e}_D = (3d/4, d)$ select the sparse region,
within which the standard 2:4 specification applies:
\begin{equation}
  \View = (3d/4, d/4, 4):(d, 4, 1), \quad \vb = (1, 1, 1), \quad \vs = (1, 1, 4)
\end{equation}

\begin{table}[ht]
  \centering
  \small
  \setlength{\tabcolsep}{3pt}
  \begin{tabular}{@{}lllll@{}}
    \toprule
    Pattern            & View $\View$                                                        & Block $\vb$                                     & Scope $\vs$       & Keep \\
    \midrule
    Unstructured       & $(M,\;K):(K,1)$                                                     & $(1,1)$                                         & $(M,\;K)$             & $k$  \\
    N:M (2:4)          & $(M,\;K):(K,1)$                                                     & $(1,1)$                                         & $(1,\;4)$             & $2$  \\
    Block $b\!\times\!b$& $(M,\;K):(K,1)$                                                    & $(b,\;b)$                                       & $(1,\;K/b)$           & $k$  \\
    Channel            & $(C_{\text{out}},\;C_{\text{in}}HW)$                                & $(1,\;C_{\text{in}}HW)$                         & $(C_{\text{out}},)$   & $k$  \\
    Head               & $(h,\;d/h,\;d)$                                                     & $(1,\;d/h,\;d)$                                 & $(h,\;1,\;1)$         & $k$  \\
    \bottomrule
  \end{tabular}
  \caption{Canonical sparsity patterns expressed in \sthree{}. Strides are implied by
    standard row-major order. The \emph{Keep} column is the number of blocks retained per scope.}
  \label{tab:patterns}
\end{table}

\subsection{Experimental Patterns: Complete Specifications}
\label[appendix]{app:exp_patterns_full}

For a weight matrix $\vW \in \R^{M \times K}$, the four patterns used in the
experiments are fully determined by the following View, Block, and Scope shapes.

\textbf{2:4 sparsity.}
\begin{equation}
  \View = \Layout_{\text{phys}}, \quad \vb = (1,1), \quad \vs = (1,4), \quad k = 2
\end{equation}
Scalar blocks are scoped into blocks of four consecutive columns; retaining
$k\!=\!2$ enforces exactly 2-of-4 sparsity per row.

\textbf{4:8 sparsity.}
\begin{equation}
  \View = \Layout_{\text{phys}}, \quad \vb = (1,2), \quad \vs = (1,4), \quad k = 2
\end{equation}
Each 2-column block is the atomic unit; blocks of four blocks span 8 consecutive
columns, and $k\!=\!2$ retains half.

\textbf{Coupled 2:4.}
A strided view pairs columns that are 8 positions apart within each 16-column
segment. Element $[m, g, i, j]$ maps to $W_{m,\; 16g + i + 8j}$:
\begin{equation}
  \View = (M, K/16, 8, 2):(K, 16, 1, 8), \quad \vb = (1,1,1,2), \quad
  \vs = (1,1,4,1), \quad k = 2
\end{equation}
Each block bundles a coupled column pair; blocks of four such pairs enforce
2:4 sparsity over paired columns jointly.  The compression ratio improves from
$9/16$ to $17/32$ (\Cref{app:coupled_24_compression}).

\textbf{16-column block sparsity.}
A strided view couples rows that are 8 apart within 16-row chunks.
Element $[p, r, c]$ maps to $W_{p + 8r,\; c}$:
\begin{equation}
  \View = (8, 2, K):(K, 8K, 1), \quad \vb = (1,1,16), \quad \vs = (1,2,1),
  \quad k = 1
\end{equation}
Each block spans 16 contiguous columns; paired rows (0 and 8, 1 and 9, \ldots)
share the same column mask, yielding 50\% block sparsity.  This pattern preserves
contiguous memory access along the GEMM reduction dimension
(\Cref{app:block_sparse_gemm}).

\section{Hardware Mapping}
\label[appendix]{app:hardware}

\subsection{Tensor Core Warp Fragment Layout}
\label[appendix]{app:wmma}

NVIDIA tensor cores execute matrix multiply-accumulate on fixed tile shapes.
The \texttt{mma.sync.aligned.m16n8k16} instruction (Ampere and later) is the
fundamental fp16 tile: it multiplies a $16\!\times\!16$ A fragment by a
$16\!\times\!8$ B fragment and accumulates into a $16\!\times\!8$ C fragment,
with all 32 threads of a warp participating simultaneously.

\Cref{fig:wmma} shows the per-thread element assignment for each matrix.
Let $t$ denote the lane index ($0 \le t < 32$), $g = \lfloor t/4 \rfloor$
(thread block, $0\!\le\!g\!<\!8$), and $\tau = t \bmod 4$ (within-block offset).

\begin{itemize}
  \item \textbf{A} ($16\!\times\!16$, row-major): thread $t$ owns m-rows
        $\{g,\; g\!+\!8\}$ and k-columns $\{\tau\!\cdot\!2,\;\tau\!\cdot\!2\!+\!1,\;
        \tau\!\cdot\!2\!+\!8,\;\tau\!\cdot\!2\!+\!9\}$, comprising two k-column pairs,
        one per k-half.
  \item \textbf{B} ($16\!\times\!8$, col-major): thread $t$ owns k-rows
        $\{\tau\!\cdot\!2,\;\tau\!\cdot\!2\!+\!1,\;\tau\!\cdot\!2\!+\!8,\;
        \tau\!\cdot\!2\!+\!9\}$ and n-column $g$.
  \item \textbf{C} ($16\!\times\!8$): thread $t$ owns m-rows $\{g,\; g\!+\!8\}$
        and n-columns $\{\tau\!\cdot\!2,\; \tau\!\cdot\!2\!+\!1\}$.
\end{itemize}

\begin{figure}[t]
  \centering
  \includegraphics[width=0.8\linewidth]{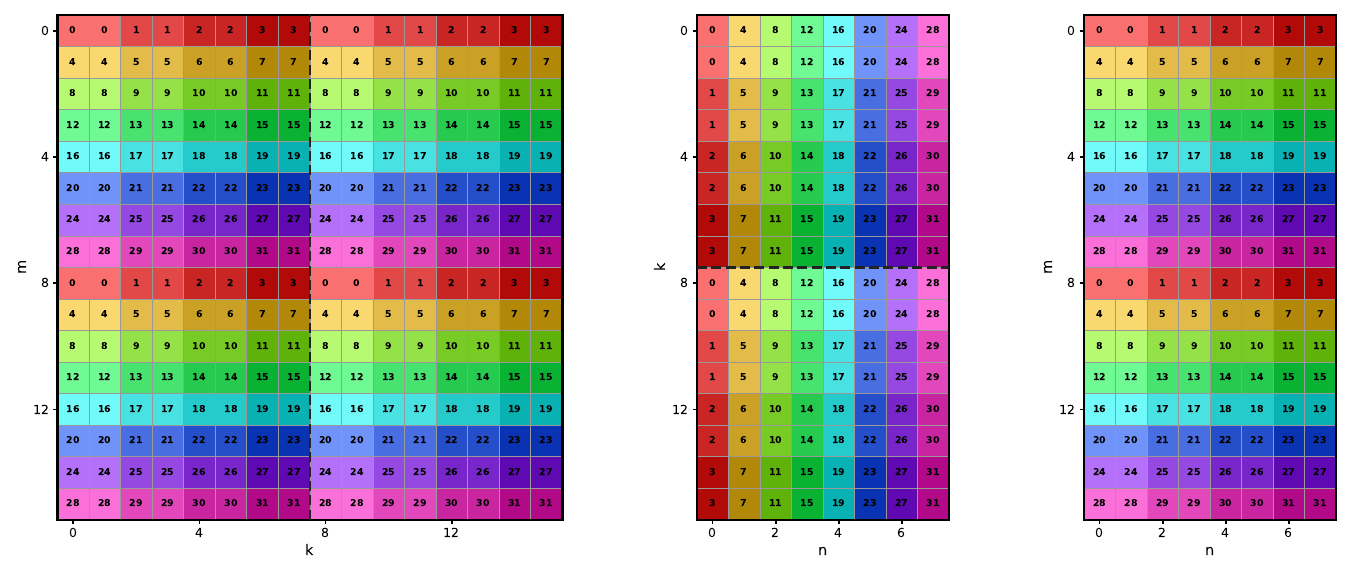}
  \caption{Thread-fragment assignment for \texttt{mma.sync.aligned.m16n8k16}.
    Cell color encodes the owning thread (hue = thread block $g = \lfloor t/4 \rfloor$,
    brightness = within-block offset $t \bmod 4$). Dashed lines mark the k-half
    boundary ($k\!=\!8$). Each thread owns 8 fp16 elements of A, 4 of B, and 4
    fp32 elements of C. The 2-column k-pair structure of A and B is the hardware
    reason why 2:4 and 4:8 sparsity patterns are natural: surviving elements fill
    exactly the k-col pairs owned by a single thread.}
  \label{fig:wmma}
\end{figure}

The 2-column granularity in A and B directly motivates our sparsity choices.
For 2:4 sparsity, each scope of 4 consecutive k-columns corresponds to exactly
two consecutive thread-owned column pairs; retaining 2 of 4 means keeping one
complete pair per thread.  For 4:8, scopes span two such pairs per k-half, and
retaining 4 of 8 keeps one pair per k-half per thread.  The 16-column block
sparsity aligns with the full k-extent of a single mma tile, so each surviving
block maps to a contiguous set of A columns owned by all 4 threads in a thread block.

\subsection{Coupled 2:4 vs.\ Standard 2:4}
\label[appendix]{app:coupled_vs_standard}

Standard 2:4 sparsity blocks 4 \emph{consecutive} k-columns and retains 2.
Each block of 4 elements requires its own 4-bit mask to record which 2 survive.

Coupled 2:4 instead treats the column \emph{pair} $\{c,\; c\!+\!8\}$ as the
atomic pruning unit and forms scopes of 4 such pairs.  Two pairs are retained
per scope, enforcing 2-of-4 sparsity over pairs rather than over individual
columns.  Because both columns in a pair share the same mask bit, the metadata
for the second k-half is eliminated entirely: one 4-bit mask covers 8 values
instead of 4.  This halves the metadata overhead from $4/32 = 12.5$\% to
$4/64 = 6.25$\% of the compressed payload.

The reduction is particularly significant for sub-byte quantized weights.
At 4-bit precision, standard 2:4 metadata adds 0.5 bits per weight
($\sim$10\% memory overhead), whereas coupled 2:4 reduces this to 0.25 bits
($\sim$5\%).  \Cref{fig:coupled_vs_standard} illustrates the difference.


\subsection{NVIDIA Sparse Tensor Cores}

NVIDIA Ampere and later architectures accelerate 2:4 structured sparsity in hardware.
The corresponding \sthree{} specification uses scalar blocks blocked into fours along
the innermost dimension:
\begin{equation}
 \View_B = (M, K/4, 4):(K, 4, 1), \quad \vb = (1, 1, 1)
\end{equation}
\begin{equation}
 \View_G = (M, K/4, 4):(1, M, M \cdot K/4), \quad \vs = (1, 1, 4)
\end{equation}

\subsection{Block-Sparse GEMM}
\label[appendix]{app:block_sparse_gemm}

Efficient block-sparse matrix multiplication requires blocks that align with GEMM
tiling. Common tile sizes are $64 \times 64$, $128 \times 128$, and $256 \times 256$,
so setting $\vb = (64, 64)$ or $(128, 128)$ ensures that each non-zero block maps
directly to a hardware tile without padding.

Column-block sparsity (e.g.\ 16-column blocks) is particularly hardware-friendly
because it preserves contiguous memory access along the reduction dimension of the
matrix multiply $\mathbf{Y} = \mathbf{X} \vW^T$. Each surviving column block corresponds to a
contiguous slice of the input activation $\mathbf{X}$, so the sparse GEMM reduces to a
sequence of dense sub-GEMMs with no gather or scatter overhead. The block width (16)
aligns with GPU vector load widths (128-bit or 256-bit), ensuring coalesced memory
transactions. Row-coupling across the view further guarantees that paired rows share
the same column mask, allowing the sparse index metadata to be stored once per pair
rather than per row.

\subsection{Coupled 2:4 Compression Ratio}
\label[appendix]{app:coupled_24_compression}

Standard 2:4 sparsity stores 2 half-precision values (32 bits) plus a 4-bit mask per
block of 4 elements, giving a compressed size of 36 bits per 64 uncompressed bits, or a
$9/16$ compression ratio.

Coupled 2:4 applies the same 2:4 pattern over column pairs. Each scope now contains 4
pairs (8 half-precision values, 128 bits uncompressed). Retaining 2 pairs stores 4
values (64 bits) with the same 4-bit mask, for a total of 68 bits. The compression
ratio is $68/128 = 17/32 \approx 0.531$, improving over the standard $9/16 \approx
0.563$. The gain arises because the 4-bit mask metadata is amortized over 8 values
instead of 4, and the paired columns share a single index structure.

\subsection{Memory Access Patterns}

Coalesced memory access on GPUs requires that threads in a warp read contiguous
addresses. In \sthree{}, this translates to designing the view so that the innermost
block dimension has stride~1, ensuring that all elements within a block are loaded in a
single coalesced transaction.

\section{Qwen3-0.6B End-to-End Results}
\label[appendix]{sec:appendix_0.6b}

We repeat the end-to-end 2:4 pruning experiment from \Cref{sec:e2e} on the smaller
Qwen3-0.6B (28 layers, $d\!=\!1024$, $d_{\text{ff}}\!=\!3072$).

\begin{table}[t]
  \centering
  \small
  \begin{tabular}{@{}lcccc@{}}
    \toprule
    Method & Dense PPL & PPL ($\downarrow$) & PPL Increase & Time \\
    \midrule
    S-OBD     & 26.13 & 369.78 & $+$1315.1\% & 79s \\
    SparseGPT & 26.13 & \textbf{126.23} & $+$383.0\% & 165s \\
    S-OBS     & 26.13 & 154.16 & $+$489.9\% & 223s \\
    \bottomrule
  \end{tabular}
  \caption{End-to-end 2:4 pruning of Qwen3-0.6B (28 layers). WikiText-2 word perplexity,
    1024 C4 calibration samples. S-OBD (no compensation) performs substantially worse
    than both SparseGPT and S-OBS, confirming that compensation is critical.}
  \label{tab:e2e_0.6b}
\end{table}

\begin{figure}[!t]
  \centering
  \begin{subfigure}[t]{0.3\linewidth}
    \centering
    \includegraphics[width=\linewidth]{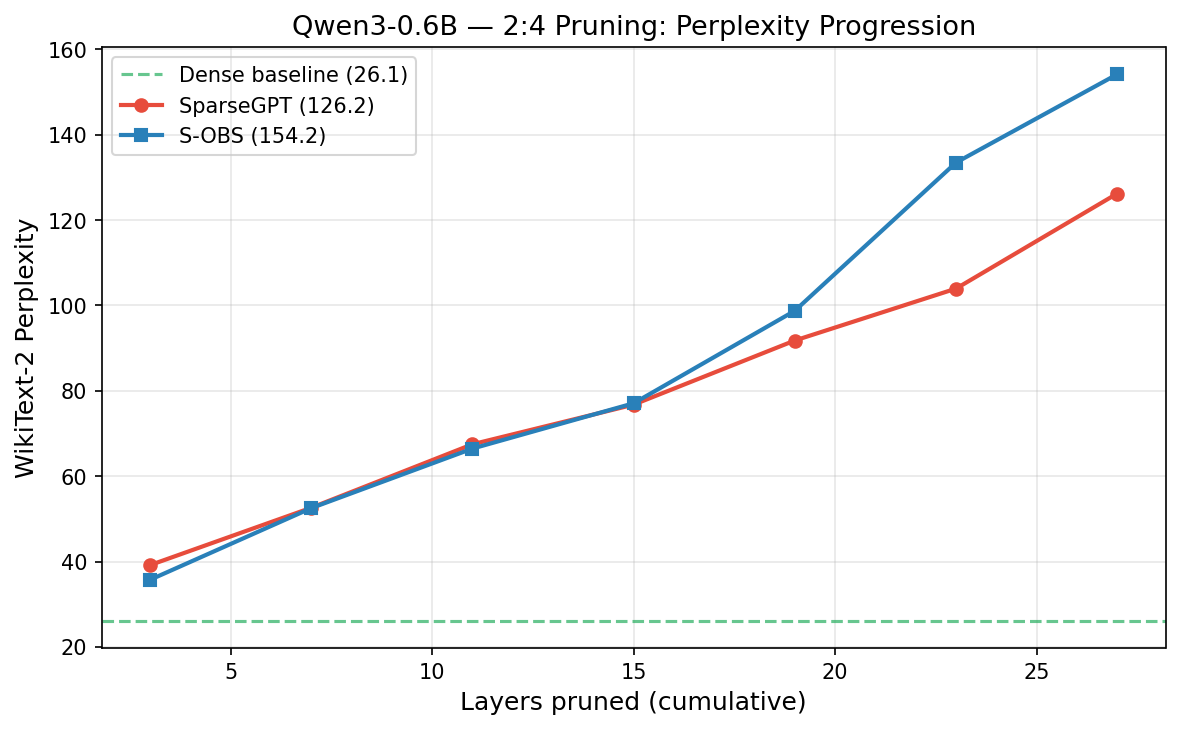}
    \caption{Perplexity progression}
    \label{fig:e2e_0.6b_ppl}
  \end{subfigure}\hfill
  \begin{subfigure}[t]{0.68\linewidth}
    \centering
    \includegraphics[width=\linewidth]{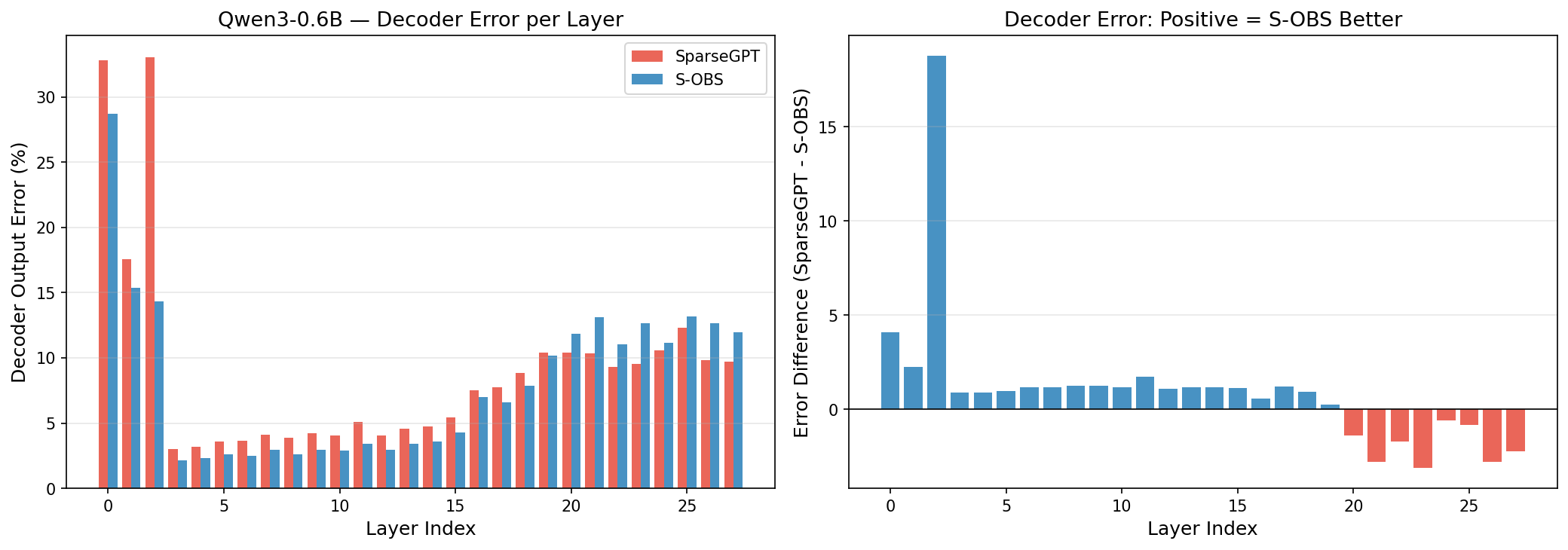}
    \caption{Decoder output error}
    \label{fig:e2e_0.6b_decoder}
  \end{subfigure}
  \caption{End-to-end 2:4 pruning of Qwen3-0.6B.
    \Cref{fig:e2e_0.6b_ppl}:~WikiText-2 perplexity after pruning every fourth decoder
    layer. SparseGPT maintains lower perplexity throughout, reaching 126.2 vs.\ 154.2
    for S-OBS.
    \Cref{fig:e2e_0.6b_decoder}:~Per-layer decoder output error. S-OBS achieves lower
    error in early-to-mid layers, but SparseGPT becomes substantially better in later
    layers.}
  \label{fig:e2e_0.6b_ppl_decoder}
\end{figure}

\begin{figure}[t]
  \centering
  \begin{subfigure}[t]{0.48\linewidth}
    \centering
    \includegraphics[width=\linewidth]{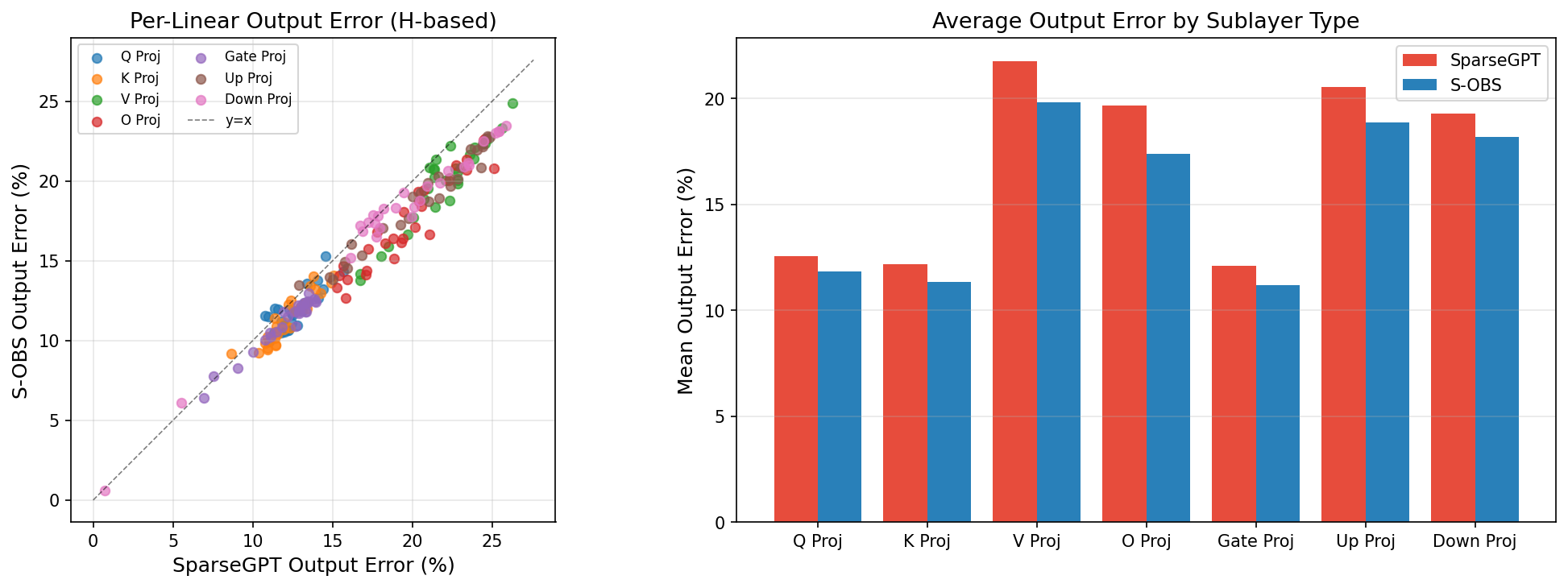}
    \caption{Per-linear output error}
    \label{fig:e2e_0.6b_linear}
  \end{subfigure}\hfill
  \begin{subfigure}[t]{0.48\linewidth}
    \centering
    \includegraphics[width=\linewidth]{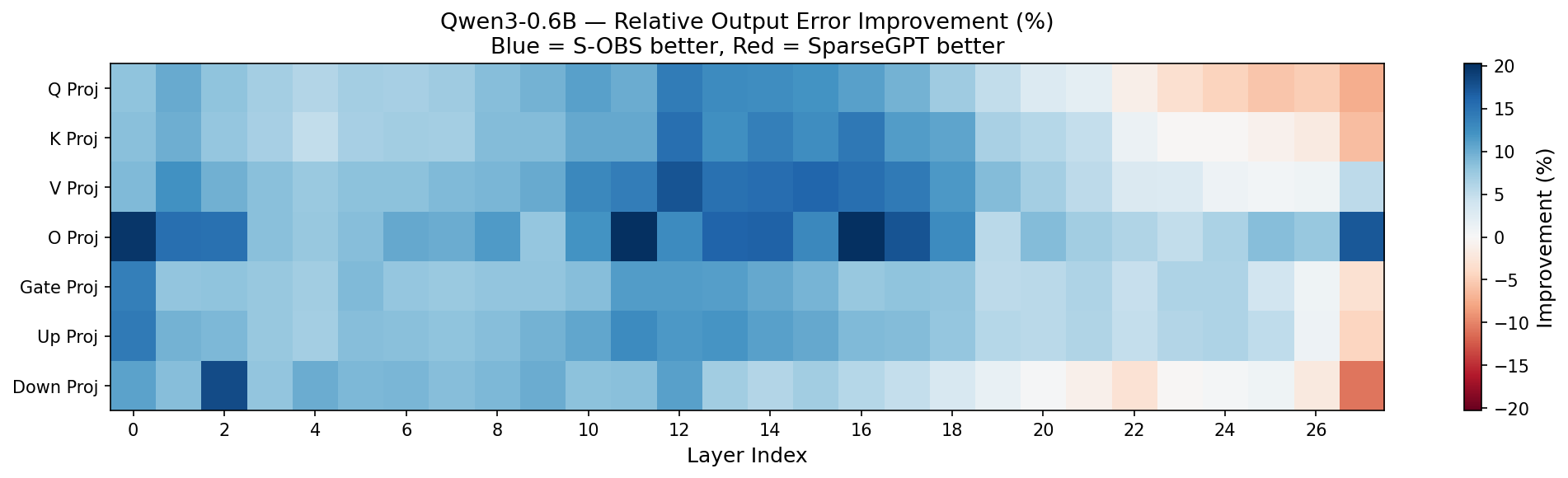}
    \caption{Improvement heatmap}
    \label{fig:e2e_0.6b_heatmap}
  \end{subfigure}
  \caption{Per-linear analysis of Qwen3-0.6B 2:4 pruning.
    \Cref{fig:e2e_0.6b_linear}:~Output error scatter. S-OBS wins on 178 of
    196 linear layers, comparable to the 177/196 advantage on Qwen3-1.7B.
    \Cref{fig:e2e_0.6b_heatmap}:~Improvement heatmap. S-OBS dominates early-to-mid
    layers but loses in the final layers, particularly for projections with large input
    dimension (\texttt{down\_proj}, $K\!=\!3072$).}
  \label{fig:e2e_0.6b_linear_heatmap}
\end{figure}

The per-linear error advantage is consistent across model sizes. S-OBS
achieves lower H-based output error on 178 of 196 linear layers
(\Cref{tab:e2e_0.6b,fig:e2e_0.6b_linear}), nearly identical to the
177/196 ratio on Qwen3-1.7B. The mean per-linear error is 15.51\% for
S-OBS vs.\ 16.88\% for SparseGPT.

However, per-linear gains do not translate to perplexity gains on small models.
Despite winning the per-linear metric on 91\% of layers, S-OBS achieves substantially
higher perplexity (154.2 vs.\ 126.2). This contrasts with Qwen3-1.7B, where S-OBS
achieves both lower per-linear errors \emph{and} lower perplexity (46.3 vs.\ 56.4).

The late-layer decoder error reversal is more severe in smaller models. The
decoder output error (\Cref{fig:e2e_0.6b_decoder}) shows that SparseGPT
overtakes S-OBS much earlier and by a larger margin than on Qwen3-1.7B. This is
consistent with the Hessian rank-deficiency explanation from \Cref{sec:e2e}: smaller
models have lower hidden dimension ($d\!=\!1024$) relative to the calibration set size
($n\!=\!1024$), so $K/n$ ratios are already at or above 1.0 for attention projections
($K\!=\!2048$ for \texttt{q\_proj}) and 3.0 for MLP projections
($K\!=\!3072$ for \texttt{down\_proj}). The compounding of late-layer errors overwhelms
the gains from early layers, leading to worse end-to-end perplexity despite better
per-linear metrics.

\section{Per-Linear-Type Error Progression}
\label[appendix]{sec:appendix_linear_type}

In both models, S-OBS (solid) achieves lower per-linear error than SparseGPT (dashed)
across nearly all projection types and layers (\Cref{fig:error_progression}).
The decoder error plots confirm that this advantage compounds in early-to-mid layers
but reverses in later layers, especially for the smaller model.

The decoder error varies substantially across layers: early layers (0--2) exhibit
high error that drops sharply, middle layers maintain low error, and late layers
show a gradual increase. This suggests that some layers are inherently more
``prunable'' than others, likely because middle layers develop more redundant
representations during pre-training, while early layers (close to the embedding)
and late layers (close to the prediction head) carry less redundancy.

However, both the pruning decisions and the error measurements are conditioned on
the C4 calibration data. Whether the same layers remain easy to prune under a
different data domain (e.g.\ code, mathematics, or multilingual text) is an open
question. The Hessian $\vH = \frac{1}{N}\mathbf{X}^T\mathbf{X}$ captures
second-order statistics of the calibration distribution; a domain shift could
change which directions are important, potentially altering both the optimal
pruning masks and the per-layer error profile. Investigating the sensitivity
of layer-wise prunability to calibration domain is left for future work.

\begin{figure}[H]
  \centering
  \begin{subfigure}[t]{0.48\linewidth}
    \centering
    \includegraphics[width=\linewidth]{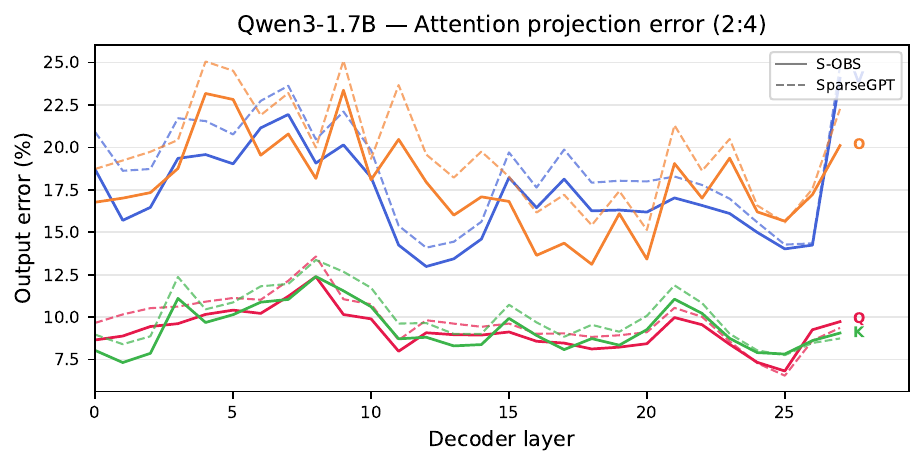}
    \caption{1.7B attention (Q/K/V/O)}
    \label{fig:attn_err_1.7b}
  \end{subfigure}\hfill
  \begin{subfigure}[t]{0.48\linewidth}
    \centering
    \includegraphics[width=\linewidth]{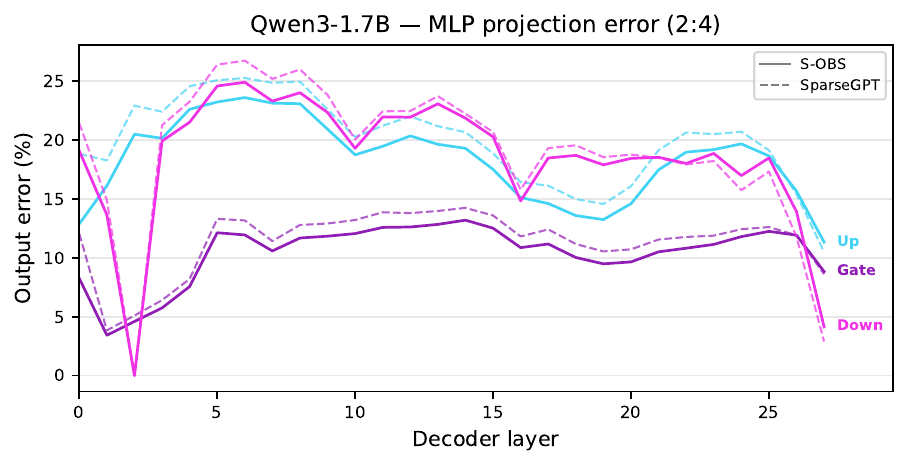}
    \caption{1.7B MLP (Gate/Up/Down)}
    \label{fig:mlp_err_1.7b}
  \end{subfigure}\\[6pt]
  \begin{subfigure}[t]{0.48\linewidth}
    \centering
    \includegraphics[width=\linewidth]{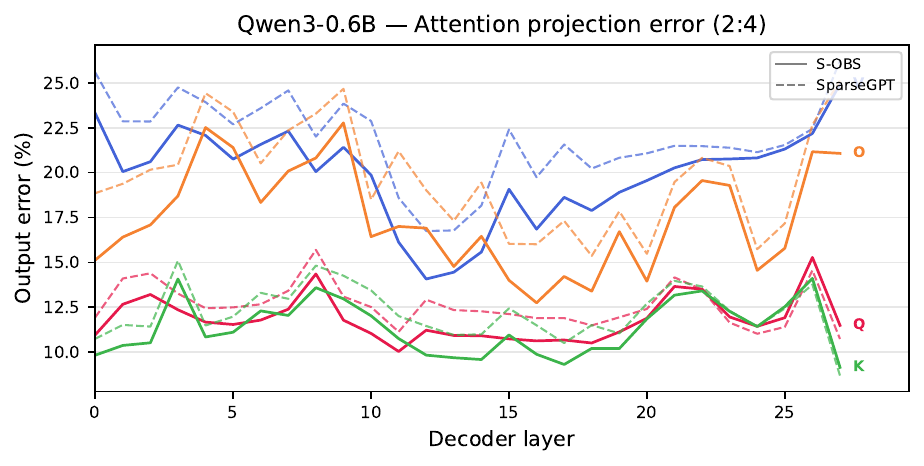}
    \caption{0.6B attention (Q/K/V/O)}
    \label{fig:attn_err_0.6b}
  \end{subfigure}\hfill
  \begin{subfigure}[t]{0.48\linewidth}
    \centering
    \includegraphics[width=\linewidth]{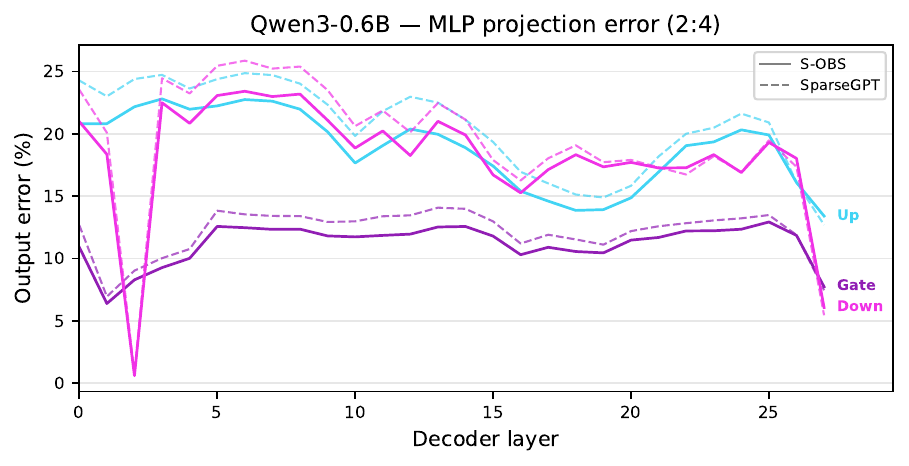}
    \caption{0.6B MLP (Gate/Up/Down)}
    \label{fig:mlp_err_0.6b}
  \end{subfigure}\\[6pt]
  \begin{subfigure}[t]{0.48\linewidth}
    \centering
    \includegraphics[width=\linewidth]{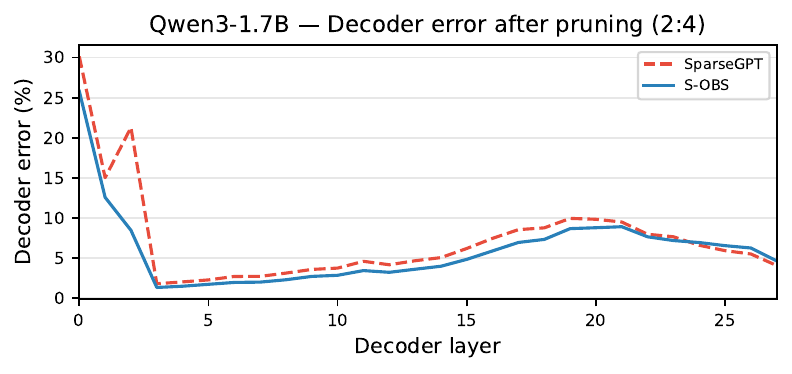}
    \caption{1.7B decoder error}
    \label{fig:decoder_err_1.7b}
  \end{subfigure}\hfill
  \begin{subfigure}[t]{0.48\linewidth}
    \centering
    \includegraphics[width=\linewidth]{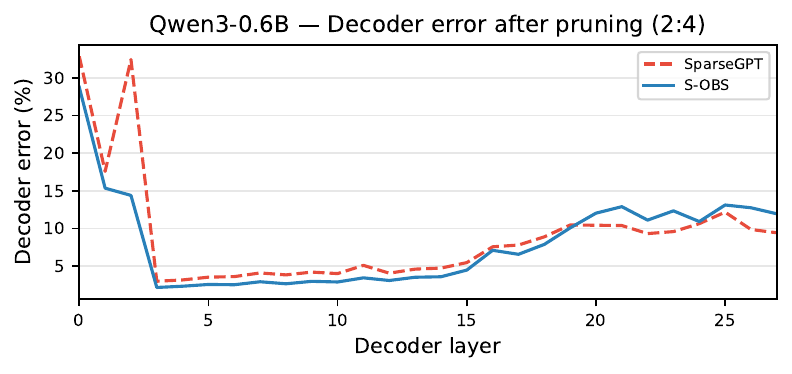}
    \caption{0.6B decoder error}
    \label{fig:decoder_err_0.6b}
  \end{subfigure}
  \caption{Per-projection output error and decoder error across decoder layers
    (2:4 pruning, 1024 C4 calibration samples). Solid = S-OBS, dashed = SparseGPT.
    Top row: attention projections (Q/K/V/O). Middle row: MLP projections (Gate/Up/Down).
    Bottom row: full decoder output error.
    Left column: Qwen3-1.7B. Right column: Qwen3-0.6B.
    \textbf{Per-linear metric} (top/middle):
    $\sqrt{\operatorname{tr}(\Delta\vW\,\vH\,\Delta\vW^T) \,/\, \operatorname{tr}(\vW\,\vH\,\vW^T)}$
    where $\Delta\vW = \hat{\vW} - \vW$, $\vH = \frac{1}{N}\mathbf{X}^T\mathbf{X}$.
    \textbf{Decoder metric} (bottom):
    $\normf{\hat{\mathbf{Y}} - \mathbf{Y}} \,/\, \normf{\mathbf{Y}}$
    where $\mathbf{Y}$ and $\hat{\mathbf{Y}}$ are the decoder layer outputs before
    and after pruning.
    S-OBS achieves lower per-linear error across nearly all layers and projection types,
    but the decoder-level advantage reverses in later layers, more severely for the
    smaller model.}
  \label{fig:error_progression}
\end{figure}

\begin{figure}[H]
  \centering
  \begin{subfigure}[t]{0.48\linewidth}
    \centering
    \includegraphics[width=\linewidth]{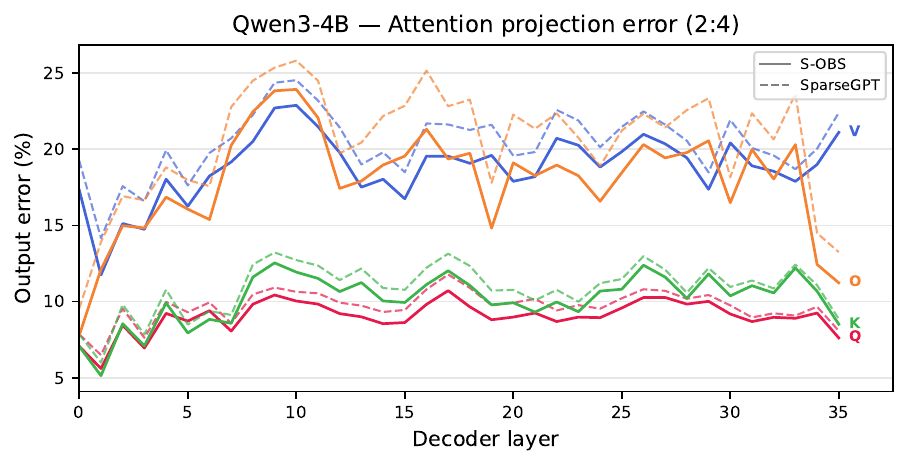}
    \caption{4B attention (Q/K/V/O)}
    \label{fig:attn_err_4b}
  \end{subfigure}\hfill
  \begin{subfigure}[t]{0.48\linewidth}
    \centering
    \includegraphics[width=\linewidth]{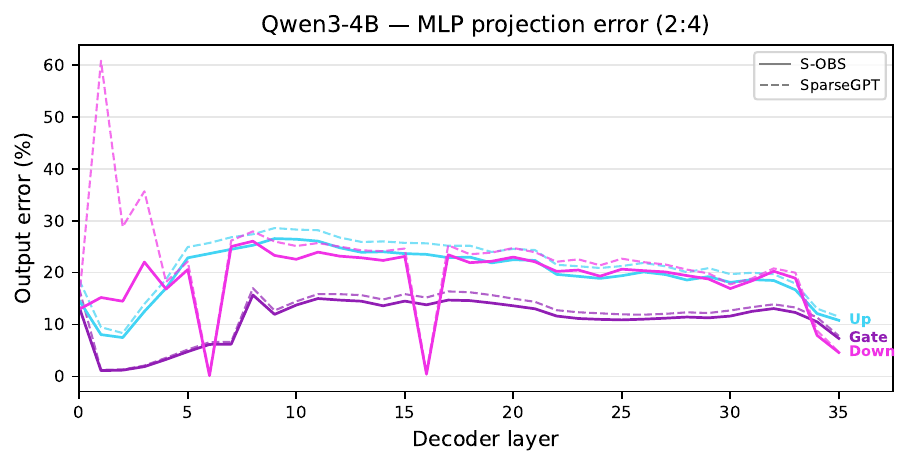}
    \caption{4B MLP (Gate/Up/Down)}
    \label{fig:mlp_err_4b}
  \end{subfigure}\\[6pt]
  \begin{subfigure}[t]{0.48\linewidth}
    \centering
    \includegraphics[width=\linewidth]{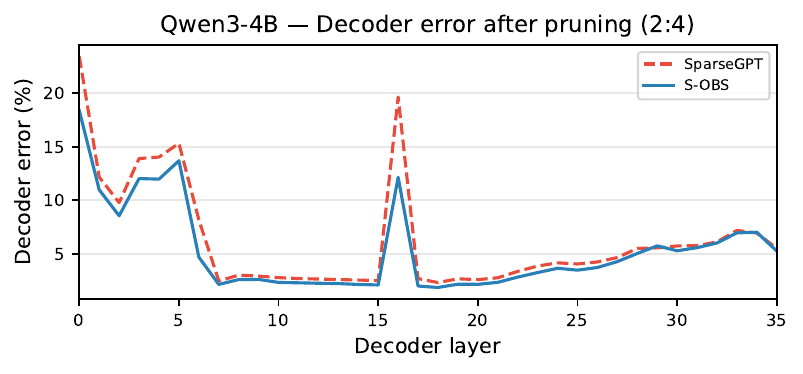}
    \caption{4B decoder error}
    \label{fig:decoder_err_4b}
  \end{subfigure}
  \caption{Qwen3-4B (36 layers): per-projection output error and decoder error
    (2:4 pruning, 1024 C4 calibration samples). Same format as
    \Cref{fig:error_progression}. S-OBS wins on all 252 linear projections.
    The decoder error spike at layer 16 is shared by both methods, suggesting
    an inherently sensitive layer.}
  \label{fig:error_progression_4b}
\end{figure}

\section{Wanda with Block Structure}
\label[appendix]{appendix:wanda}

Wanda~\cite{sun2023wanda} scores each element by combining weight magnitude with input activation norms:
\begin{equation}
  S_{ij}^{\text{Wanda}} = |W_{ij}| \cdot \| \mathbf{X}_{:,j} \|_2
\end{equation}
The block saliency is $S_k^{\text{Wanda}} = \sum_{(i,j) \in \Elements(k)} S_{ij}^{\text{Wanda}}$,
and scope-level pruning applies $k$-sparsity over block scores. This is equivalent to S-OBD
with a diagonal Hessian approximation $H_{jj} \approx \|\mathbf{X}_{:,j}\|_2^2 / N$ and no
weight compensation.

\section{Implementation}
\label[appendix]{appendix:implementation}

We release \texttt{sparsekit}, a PyTorch library that implements \sthree{}.
The documentation is bundled with the submission and can be browsed locally
(see the included \texttt{README.md}).

\subsection{Library Architecture}

\texttt{sparsekit} is organized in three layers that correspond directly to
the three \sthree{} primitives, plus a pruners layer on top.

\paragraph{Layer 1: View (\texttt{sparsekit.view}).}
\texttt{View(param, shape, stride)} is a lightweight dataclass that wraps an
\texttt{nn.Parameter} with an arbitrary $n$-dimensional shape-stride pair.
Data access is via \texttt{torch.as\_strided}, which constructs a zero-copy
view into the parameter's storage; there is no memory allocation.  Crucially,
writes to \texttt{view.data} propagate directly back to the underlying
parameter via the \texttt{@data.setter}, so all in-place pruning operations
are transparent to the rest of the model.

The class exposes two key constructors.  \texttt{View.from\_existing(param)}
wraps a plain \texttt{Parameter} with its physical layout (identity view).
\texttt{View(param, shape, stride)} constructs an arbitrary strided view; the
caller is responsible for ensuring the shape-stride pair is consistent with the
parameter's storage.  The \texttt{View} is the only place where the physical
layout appears; all higher-level code is agnostic to it.

\paragraph{Layer 2: Block Specification (\texttt{sparsekit.block}).}
\texttt{BlockSpec(view, shape)} defines the atomic sparsity unit.  The block
shape $\vb$ must divide the view shape element-wise; the library validates this
at construction time.  The \emph{block grid} (the array of block indices) has
shape $\vs^{(B)} = (s_k / b_k)_{k=0}^{n-1}$ and is computed lazily as a
\texttt{cached\_property}.

\texttt{BlockSpec} provides a rich set of reduction operations (norms, min,
max) that operate over block elements, and index utilities that map between
block grid coordinates and physical parameter positions.  These index mappings
are the key bridge between the abstract specification and actual tensor
indexing in the pruner.

\texttt{BlockCoupling} holds a list of \texttt{BlockSpec}s and exposes them as
a single virtual block whose element set is the union of all member blocks at
the same grid coordinate (after applying the per-tensor permutation).  This
implements \Cref{def:coupled_block}.

\paragraph{Layer 3: Scope Specification (\texttt{sparsekit.scope}).}
\texttt{ScopeSpec(block, shape)} operates \emph{on the block grid}: blocks
are to blocks what elements are to blocks.  The scope shape $\vs$ divides the
block grid shape, yielding a scope grid of shape $(\vs^{(B)}_k / g_k)$.

The key operation is \texttt{hard\_threshold}, which implements $k$-sparsity
within each scope in a fully vectorized manner: it computes a per-scope
$k$-th largest threshold, then zeros out all blocks below that threshold via
an in-place mask applied through the \texttt{View.apply\_mask} write-through
mechanism.  No gather/scatter is needed because the block view is contiguous
in the scope's logical layout.

\texttt{ScopeCoupling} mirrors \texttt{BlockCoupling} at the scope level,
concatenating block saliencies across tensors before ranking.

\paragraph{Layer 4: Pruners (\texttt{sparsekit.pruners}).}
\texttt{StructuredOBS(scope\_spec, H)} implements Algorithm~\ref{alg:scope_obs}
on top of any \texttt{ScopeSpec}.  It accepts the empirical Hessian
$\vH = \frac{1}{N}\mathbf{X}^T\mathbf{X}$ and pre-computes $\vH^{-1}$ (with
damping) once.  Per-row pruning proceeds by:
\begin{enumerate}
  \item Extracting the per-row sub-matrix $C_m \in \R^{K \times K}$ as an
        fp16 clone of $\vH^{-1}$ (halves memory bandwidth).
  \item Scoring each block by $S_j = \frac{1}{2}\vw_j^T [C_m]_{jj}^{-1} \vw_j$.
  \item Zeroing the $|\Scope|_B - k$ lowest-saliency blocks and applying the
        Schur update $C_m \mathrel{-}= C_m[:,I_j] [C_m]_{jj}^{-1} C_m[I_j,:]$
        via \texttt{torch.addmm} (fused, no temporary allocation).
\end{enumerate}
Row-coupled patterns (e.g.\ 16-column block) are handled by a separate code
path that operates on multi-row view chunks and uses sequential Schur updates
with $\mathtt{ng}=1$ to avoid cross-row contamination.
\subsection{Dynamic Sparse Training}
\label[appendix]{appendix:dynamic_sparse}

The experiments in the main paper focus on zero-shot (post-training) pruning, where a
pre-trained model is sparsified without retraining. \sthree{} also supports dynamic sparse
training, where the sparsity mask evolves during training while maintaining a fixed
sparsity budget.

In dynamic sparse training~\cite{mocanu2018scalable,evci2020rigging}, the training loop
alternates between standard gradient updates on the active (non-zero) parameters and
periodic mask updates that prune low-saliency connections and regrow new ones. \sthree{}
provides the structural backbone for this process: the View, Block, and Scope
specifications define which parameters compete for retention, the saliency criteria
(S-OBD or S-OBS) determine which blocks to prune, and the regrowth step activates new
blocks within the same block structure.

The regrowth policy can target blocks with high gradient magnitude (RigL~\cite{evci2020rigging})
or random blocks (SET~\cite{mocanu2018scalable}); in both cases the block-level saliency scores
from \Cref{sec:integrations} serve directly as the pruning criterion, and the Scope constraint
ensures that the budget $k$ is respected after every update step.
Dynamic sparse training with \sthree{}-structured masks is left for future work.

\subsection{Expressing the Four Experimental Patterns}

\begin{lstlisting}[style=pythonstyle]
from sparsekit import View, BlockSpec, ScopeSpec, StructuredOBS

# 2:4 sparsity  (W: M x K, keep 2 of 4 consecutive columns)
v     = View.from_existing(W)
block = BlockSpec(v, shape=(1, 1))
part = ScopeSpec(block, shape=(1, 4))

# 4:8 sparsity  (2-column blocks, keep 2 of 4 blocks)
v     = View.from_existing(W)
block = BlockSpec(v, shape=(1, 2))
part = ScopeSpec(block, shape=(1, 4))

# Coupled 2:4  (pair columns 8 apart within each 16-col segment)
v     = View(W, shape=(M, K//16, 8, 2), stride=(K, 16, 1, 8))
block = BlockSpec(v, shape=(1, 1, 1, 2))
part = ScopeSpec(block, shape=(1, 1, 4, 1))

# 16-column block  (rows 8 apart share one column mask)
v     = View(W, shape=(8, 2, K), stride=(K, 8*K, 1))
block = BlockSpec(v, shape=(1, 1, 16))
part = ScopeSpec(block, shape=(1, 2, 1))

# Pruning: same call for all patterns
H   = (X.T @ X) / X.shape[0]          # empirical Hessian (K x K)
obs = StructuredOBS(part, H)
obs.prune_true_obs(nnz=k)              # in-place, writes through View
\end{lstlisting}

The identical \texttt{prune\_true\_obs} call works for all four patterns; all
structural differences are encoded in \texttt{part}.  This is the concrete
realization of the claim that \sthree{} decouples the sparsity structure from
the pruning algorithm.
